\documentclass[conference]{IEEEtran}
\IEEEoverridecommandlockouts
\usepackage{cite}
\usepackage{amsmath,amssymb,amsfonts}
\usepackage{algorithmic}
\usepackage{graphicx}
\usepackage{textcomp}
\usepackage{xcolor}
 
\usepackage{times}
\usepackage{soul}
\usepackage{url}
\usepackage{subcaption}
\usepackage[hidelinks]{hyperref}
\usepackage[utf8]{inputenc}
\usepackage{graphicx}
\usepackage{amsmath}
\usepackage{amsthm}
\usepackage{booktabs}
\usepackage{algorithm}
\usepackage{algorithmic}
\usepackage[switch]{lineno}
\usepackage{thm-restate} 
\usepackage{xspace,amsfonts,amsmath,amsthm}

\newtheorem{example}{Example}

\def\BibTeX{{\rm B\kern-.05em{\sc i\kern-.025em b}\kern-.08em
    T\kern-.1667em\lower.7ex\hbox{E}\kern-.125emX}}

\usepackage[nolist]{acronym}
\begin{acronym}[DRAM]
	\acro{TA}{Tsetlin automaton}
	\acro{SSL}{Stochastic Searching on the Line}
	\acrodefplural{TA}{Tsetlin automata}
	\acro{TAT}{Tsetlin automaton team}
	\acro{TsM}{Tsetlin machine}
	\acro{RTM}{regression Tsetlin machine}
\end{acronym}
\usepackage{misc}

\newcommand{\TsM}{\ensuremath{\text{TsM}}\xspace}
\usepackage{todonotes}

\newcommand{\dataset}{\ensuremath{S}\xspace}
\newcommand{\variables}{\ensuremath{V}\xspace}
\newcommand{\target}{\ensuremath{T}\xspace}

\newcommand{\clauses}{\ensuremath{n}\xspace}
\newcommand{\states}{\ensuremath{N}\xspace}
\newcommand{\voting}{\ensuremath{v}\xspace}

\newcommand{\precision}{\ensuremath{s}\xspace}
\newcommand{\encsqc}{\ensuremath{\rho}\xspace}
\newcommand{\encts}{\ensuremath{\varrho}\xspace}
\newcommand{\hamming}{\ensuremath{H}\xspace}

\newcommand{\newvar}[2]{\ensuremath{v_{#1,#2}}\xspace}
\newcommand{\notrob}{\ensuremath{{\sf NotRob}}}

\newcommand{\notsim}{\ensuremath{{\sf NotSim}}}

\usepackage{amsthm}
\usepackage{pifont}

\newtheorem{definition}{Definition}

\begin{document}

\title{Verifying Properties of Tsetlin Machines
}

\author{\IEEEauthorblockN{  Emilia Przybysz}
\IEEEauthorblockA{\textit{University of Bergen} \\
Bergen, Norway \\
\url{Emilia.Przybysz@student.uib.no} }
\and
\IEEEauthorblockN{  Bimal Bhattarai}
\IEEEauthorblockA{\textit{University of Agder} \\
Agder, Norway \\
\url{bimal.bhattarai@uia.no}}
\and
\IEEEauthorblockN{  Cosimo Persia}
\IEEEauthorblockA{\textit{University of Bergen} \\
Bergen, Norway \\
\url{cosimo.persia@uib.no}}
\and
\IEEEauthorblockN{  Ana Ozaki}
\IEEEauthorblockA{\textit{University of Bergen} \\
Bergen, Norway \\
\url{ana.ozaki@uib.no}}
\and 
\IEEEauthorblockN{$\quad\quad\quad\quad\quad\quad\quad\quad\quad\quad$  Ole-Christoffer Granmo}
\IEEEauthorblockA{$\quad\quad\quad\quad\quad\quad\quad\quad\quad\quad$\textit{University of Agder} \\
$\quad\quad\quad\quad\quad\quad\quad\quad\quad\quad$ Agder, Norway \\
$\quad\quad\quad\quad\quad\quad\quad\quad\quad\quad$ \url{ole.granmo@uia.no}}
\and
\IEEEauthorblockN{  Jivitesh Sharma}
\IEEEauthorblockA{\textit{University of Agder} \\
Agder, Norway \\
\url{jivitesh.sharma@uia.no}}
}

\maketitle

\begin{abstract}
Tsetlin Machines (TsMs) are a promising and interpretable 
machine learning method which can be applied for various classification tasks. 
We present an exact encoding of TsMs into propositional logic
and formally verify properties of TsMs  using a SAT solver. In particular, we introduce in this work a notion of similarity of machine learning models and apply our notion to check for similarity of TsMs.
We also consider notions of robustness and equivalence from the literature and adapt  them for TsMs.
Then, we show the correctness 
of our encoding and provide results for the properties: adversarial   robustness,  equivalence, and similarity of TsMs.  
  In our experiments, 
we employ the  MNIST  and IMDB datasets for (respectively) image and sentiment classification. 
We discuss the results for verifying robustness obtained with TsMs with those in the literature obtained with Binarized Neural Networks on MNIST.
\end{abstract}


\section{Introduction}
Tsetlin Machines (TsMs) \cite{granmo2018tsetlin} have recently demonstrated competitive accuracy, learning speed, low memory, and low energy 
footprint 
on several tasks, 
including tasks related to 
image 
classification~\cite{granmo2019convtsetlin,sharma2021dropclause},  natural language classification~\cite{DBLP:conf/ijcai/Yadav0GG22,saha2021disc,berge2019text,yadav2021sentiment},
speech processing~\cite{lei2021kws},
spanning tabular data 
\cite{DBLP:journals/access/AbeyrathnaGG21,wheeldon2020learning},
and regression tasks. 
TsMs  are less prone to overfitting, as the training algorithm does not rely on 
minimising an error function.
Instead, the algorithm uses frequent pattern mining and resource allocation principles to extract common patterns from the data.

Unlike the intertwined nature of pattern representation in neural networks, 
TsMs decompose problems into self-contained patterns that are expressed using 
\emph{monomials}. That is,
a multiplication of Boolean variables (or their negations), 
also called \emph{conjunctive clauses}~\cite{granmo2018tsetlin}.
The 
self-contained patterns are combined to form 
classification decision through a majority vote, akin to logistic regression, however, with binary weights and a unit step output function.  
The monomials are used to build an output formula for TsMs.
This formula provides an interpretable explanation of the model,
which 
is useful to check if decisions are unfair, biased, or 
erroneous. 

\begin{example}\upshape
TsMs trained for sentiment analysis can create monomials
such as~\cite{DBLP:conf/ijcai/Yadav0GG22}:
\begin{align*}
	C^{+}(1)&={\sf truly}\cdot\overline{{\sf depressive}}\cdot\overline{ {\sf long}}\cdot\overline{ {\sf worst}}\\
	C^{+}(2)&=\overline{ {\sf boring}}\cdot\overline{ {\sf worst}}\cdot\overline{ {\sf long}}\cdot\overline{ {\sf pretentious}}\\
	C^{-}(1)&=\overline{{\sf excellent}}\cdot\overline{ {\sf good}}\cdot\overline{ {\sf like}}\\
	C^{-}(2)&=\overline{ {\sf friendly}}\cdot\overline{ {\sf charming}}\cdot\overline{ {\sf fascinating}}
\end{align*}
where $C^+(i)$ monomials are associated with 
positive sentiment and $C^-(i)$ monomials
are negative. 
Inputs that trigger more monomials of a
certain class will be classified as such. 
E.g., the comment 
``How truly, friendly, charming, and cordial is this unpretentious old serial''
triggers both positive monomials of our toy example but only one negative, namely $C^-(1)$, and, by majority vote, it is classified as positive.
\end{example}

However, current 
machine learning approaches (explainable or not) may not be robust, meaning that 
small amounts of noise  can make the model change the classification
in unexpected and uncontrolled ways~\cite{DBLP:journals/tec/SuVS19}. 
Non-robustness 
poses a real threat to the applicability of machine learning models.
This is of critical consideration because 
adversarial attacks can  weaken malware detection systems~\cite{10.1007/978-3-319-66399-9_4},
alter commands entered by users through speech recognition software~\cite{DBLP:journals/vlsisp/XieLSLCY21},
pose a security threat for systems that use computer vision~\cite{9059341},
such as self-driving cars, identity verification, medical-diagnosis, among others.

The potential lack of robustness 
of machine learning models has driven the research community into finding strategies to formally verify 
properties of such models. 
A desired robustness property (e.g.,  if the model would change the classification of a binary image
if $k$ bits are flipped) 
is not met 
if we are able to find a counterexample for
it. The problem is that such counterexamples are spread 
in a vast space of possible examples.
The first strategy to formally verify the robustness of machine learning models
via an exact encoding into a SAT problem
is the one proposed by Narodytska et al. for  binarized (deep) neural networks (BNNs)~\cite{DBLP:conf/aaai/NarodytskaKRSW18}.
This formal logic-based verification approach gives
solid guarantees that corruption will not change the 
classification (up to a predefined upper bound on the number of  corrupted bits, we are 100\% sure the classification will not change for a given verified dataset). 
This
is not possible using pure machine learning based approaches.
The authors present an exact encoding of a trained BNN 
into propositional logic.
Once the robustness property is converted into a SAT problem, one can
formally verify robustness 
using SAT solvers. 

Our work is the first work that  
investigates robustness properties of TsMs using a SAT solver.
Checking robustness via an exact encoding provides formal guarantees of the model
to 
adversarial attacks. One can also 
study relations between  meta-parameters of the model, 
compare different learned models, 
and detect classification errors.
%
In particular, 
we provide an exact 
encoding into propositional logic that 
captures 
the classification of TsMs and
leverage the capability of modern automated reasoning procedures 
to explore large search spaces and check for 
similarity between models.

In this work
we check  {adversarial robustness}, equivalence, and {similarity}.
In the mentioned work~\cite{DBLP:conf/aaai/NarodytskaKRSW18},  adversarial robustness is tested for BNNs 
but they do not provide results for 
equivalence and do not consider similarity. 
Our results indicate that TsMs provide competitive 
accuracy results and robustness results when compared to BNNs on tested datasets.  
We test the mentioned properties of TsMs using the   MNIST~\cite{mnist} and the IMDB~\cite{DBLP:conf/acl/MaasDPHNP11} datasets. The results are promising, however, the time consumed for checking robustness increases exponentially on the number of parameters in the worst case (this is an unavoidable shortcoming of formal verification based on SAT solvers since the SAT problem is an NP-hard problem), which challenges the scalability of the approach.   
In the following, we first present, in Section~\ref{sec:preli}, the TsM learning approach
and
provide basic definitions 
that are needed to understand the rational behind the encoding.
In Section~\ref{sec:lr}, we explain how to encode  a TsM
in a propositional formula and
we show the method used to check for robustness,  equivalence, and similarity.
In Section~\ref{sec:exp},
we empirically evaluate the approach and we conclude in Section~\ref{sec:conclusion}. The full version of our paper, with omitted proofs and an appendix about TsMs, is available at
\url{https://arxiv.org/abs/2303.14464}.

\section{Definitions and notations}\label{sec:preli}
We provide basic notions of propositional logic and
the TsM learning and classification algorithm
required to understand how to prove properties of
TsMs. 

\subsection{Propositional Logic and Vectors}
We use standard propositional logic 
formulas to define the SAT encoding of TsMs
and the formulas for checking robustness. 
In our notation, we write
$ \V $ 
for  a finite set of \emph{Boolean variables}, used to 
construct our propositional formulas. 
Every variable in $ \V $ is a propositional formula over \V.
We  omit `over \V' since all propositional formulas we speak of are formulated
using (a subset of) symbols from \V. 
For propositional formulas $\phi$ and $\psi$,
the expressions $(\phi\wedge\psi)$ and $(\phi\vee\psi)$
are propositional formulas, 
the \emph{conjunction} and the \emph{disjunction} of $\phi,\psi$, respectively.
Also, if $\phi$ is a propositional formula then  $\neg \phi$
(the \emph{negation} of $\phi$)
is a propositional formula. 
The semantics is given by \emph{interpretations}, as usual in propositional logic. 
They map each  variable in \V to either ``true'' ($1$) or ``false'' ($0$). 
For a   formula $\phi$, we write $\phi_{[\vec{x}\rightarrow {\Imc}]}$ 
for the result of replacing  each $x\in\vec{x}$  
in $\phi$ 
by $\Imc(x)$.  



TsMs are trained on classified binary vectors in 
the  $n$-dimensional space, with $n< |\Vsf|$.  
In our work, it is useful to talk about interpretations and their vector representation
interchangeably.
The mapping from interpretations to vectors is defined as follows.
We assume a total order 
on the elements of \Vsf. 
Given an interpretation \Imc over 
\Vsf, the vector representation of \Imc
in the $n$-dimensional space is of the form
$[\Imc(x_1),\ldots,\Imc(x_n)]$ (note that  \V can have 
more variables).
Also, we
write $ \Imc[i] $
for the value $\Imc(x_i)$ of the $i$-th variable $x_i$.
A (binary) \emph{dataset} is a set of elements of the form
$ (\Imc,y) $,
where $y$ (the classification label of \Imc) is either $0$ or $1$ and \Imc is an interpretation, treated as a vector
in the $n$-dimensional space. 
The \emph{$n$-hamming distance} of two interpretations
$\Imc $ and $  \Jmc $ over 
\Vsf,  denoted $ \hamming_n(\Imc,\Jmc) $,
is 
the sum of differing values between \Imc and \Jmc.

\subsection{The Tsetlin Machine}\label{sec:tsms}
We now present the main notions for the TsM algorithm (Algorithm~\ref{a:tsm}). 
The algorithm is based on the notion of Tsetlin Automata (TA)~\cite{Tsetlin1961}. 
A TA is a simple finite automaton 
that performs $ 2 $ actions and updates the
current state according to positive or negative feedback.  We assume the initial state to be the exclude state that is closer to the center. TAs have no final states~\cite{Tsetlin1961}. 
The shaded area in Figure~\ref{f:tsautomaton} shows a TA with 
$4$ states. If a positive feedback is received,
the automaton shifts to a state on the right and 
performs the action labelled in the current state,
otherwise it shifts to a left state.
The TsM algorithm uses a collection of TAs.
Each single TA is represented by just an integer in memory and the {TsM}
learning algorithm (explained later)
modifies
the value of the integer, based on its input.
Each  TA  votes for a specific pattern, inclusion or exclusion of variables 
(the number of variables matches with the  dimension of the TsM input vector). 

An example of a set of TAs voting for a pattern is depicted in Figure~\ref{f:tsautomaton}.
For binary classification, the TsM algorithm divides 
the set of $n$ TAs into two, denoted $C^+$ and $C^-$.
The  TAs in the set $C^+$ vote for a positive label of the input 
while the TAs in $C^-$ vote negatively.
%
%
For $k\in\{+,-\} $,
and $ 1\leq j \leq n/2 $,
we denote by $ C^k(j) $ the pattern recognised by the $ j $-th set of TAs
belonging to  group $ C^k $. Each pattern $ C^k(j) $ is represented as 
a \emph{monomial}, that is a multiplication of 
Boolean variables or their negation. For a Boolean variable
$x$ we denote its negation with  $\overline{x} $,
which is the same as $1-x$.
In Figure~\ref{f:tsautomaton}, the monomial is $ x_1 \overline{x_2}  $.
Monomials can be represented as a conjunction of variables
or their negation (e.g., the monomial $ x_1 \overline{x_2}  $
corresponds to $x_1\wedge \neg x_2$).
In previous works on TsMs (e.g.,~\cite{granmo2018tsetlin}), 
the authors have used sums and disjunctions interchangeably.
However, in our work, it is useful to distinguish between the two representations
because the sum of multiple non-zero monomials is 
 counted using the sequential counters in 
Subsection~\ref{sec:satencoding}. That is, $1+1$  is counted as $2$,
which is not the same as $1\vee 1$, evaluated to ``true'' (in symbols, $1$).
It is possible to have a literal and its negation in a monomial. This happens when the training data has some kind of inconsistent information~\cite{9445039}. 

The procedure for training a TsM is given by Algorithm~\ref{a:tsm}.
It initialises $ n $  (num. of monomials) teams of TAs (one for each variable $ \variables $ and its negation) 
with $ \states $ states per action (Line~\ref{a:tsm:3}).
Then, for each training example (Line~\ref{a:tsm:5}),
it loops for every TA team and apply the respective feedback type in order to update the recognised pattern.
The value computed in Line~\ref{a:tsm:7} is used to guide the 
randomised selection for the type of feedback to give to the TAs.
The function $ {\sf clip} $ is used to bound the sum between the interval $ [-\target,\target] $.
Finally, the algorithm returns	the functions $C^+,C^-$.	
We denote with $\Mmc(\dataset,\variables,\clauses,\states,\target,\precision)$  the  
formula
\begin{equation}\label{eq:polynomial}
\sum^{n/2}_{j=1} C^-(j) \geq \sum^{n/2}_{j=1} C^+(j)
\end{equation}
built from the functions  $C^+,C^-$  returned by 
Algorithm~\ref{a:tsm} with \dataset, \variables, \clauses, \states, \target, and \precision as input. 
We may omit $(\dataset,\variables,\clauses,\states,\target,\precision)$ if this is clear from the context
and call a expression \Mmc in this format a \emph{TsM formula}. 
Given an interpretation $\Imc$, we write 
$\Mmc(\Imc)=0$---meaning that the classification of \Imc is \emph{negative}---if 
the result of replacing each variable $x$ in the formula \Mmc by
$\Imc(x)$ results in an expression where Equation~\ref{eq:polynomial} holds, otherwise
$\Mmc(\Imc)=1$, that is, the classification of \Imc is \emph{positive}.
We 
say that \Imc is 
a positive (resp. negative) example for \Mmc 
if $\Mmc(\Imc)=1$ (resp. $\Mmc(\Imc)=0$).
\begin{example}\label{ex:1}\upshape
Assume the \TsM formula \Mmc is \[ x_1 \overline{x_2} + \overline{x_1} x_2 \geq x_1 x_2 + \overline{x_1}\ \overline{x_2}.\]
Then, e.g., 
$C^-(1)=x_1 \overline{x_2}$,   and 
$C^+(2)=\overline{x_1} \ \overline{x_2}$. 
If $ \Imc(x_1)=1 $ and $ \Imc(x_2)=0 $, after substituting the values in the formula, we have
$ 1 + 0 \geq 0 + 0 $. So, the TsM classifies $ \Imc $ as $0$. In symbols, 
$\Mmc(\Imc)=0$.
\end{example}

Each $ C^k(j) $ is trained by receiving two  feedback types given an input $ \Imc $ with a classification label. 
In short, Feedback I is given to the TAs which vote for a classification matching with the label $y$ of the input $\Imc$. It aims at increasing the number of clauses that correctly evaluates 
a positive input to true.
On the other hand,  Feedback II is given to the TAs which \emph{do not} vote for a classification matching with the label $y$ of the input $\Imc$.
It aims at combating false positive output by 
increasing the discrimination power of the clauses.
Granmo et al.  proves theoretical guarantees
regarding these feedbacks~\cite{granmo2018tsetlin}. 
To perform multiclass classification on TsMs, one  defines a TsM for each target class and decides the class by taking the TsM which classifies positively with the highest difference between positive and negative monomials.

\usetikzlibrary{arrows,automata}
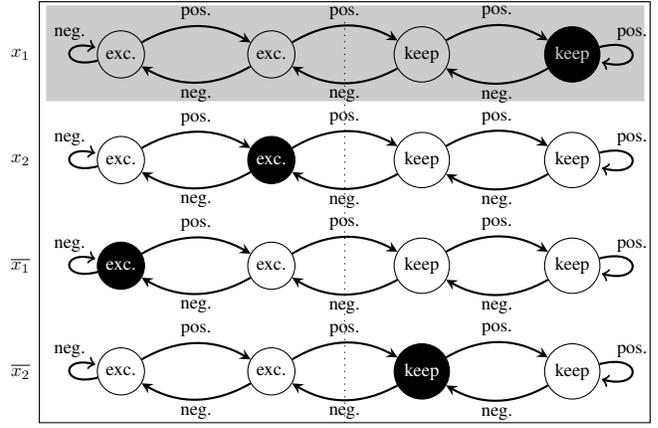
\begin{figure}
\begin{tikzpicture} [scale=0.7, every node/.style={scale=0.7},node distance = 2cm, on grid, auto]
	
	\def\ex{exc.}
	\def\keep{keep}
	\def\vardist{38pt}
	\node (q0) [state] at(0,0) {\ex};
	\node (q1) [state, right = of q0] {\ex};
	\node (q2) [state, right = of q1] {\keep};
	\node (q3) [state, right = of q2,fill=black,text=white] {\keep};
	
	\def\pos{pos.}
	\def\neg{neg.}
	\path [-stealth, thick]
	(q0) edge [loop left] node[above,yshift=4pt]{\neg}()
	(q0) edge[bend left] node{\pos}   (q1)
	(q1) edge[bend left] node{\pos}   (q2)
	(q2) edge[bend left] node{\pos}   (q3)
	(q3) edge[bend left] node{\neg}   (q2)
	(q2) edge[bend left] node{\neg}  (q1)
	(q1) edge[bend left] node{\neg}  (q0)
	(q3) edge [loop right]  node[above,yshift=4pt] {\pos}()
	;
	
	\node [left = \vardist of q0] {$ x_1 $};
	
	\node (q0) [state, below = 40pt of q0] {\ex};
	\node (q1) [state, right = of q0,fill=black,text=white] {\ex};
	\node (q2) [state, right = of q1] {\keep};
	\node (q3) [state, right = of q2] {\keep};
	
	\path [-stealth, thick]
	(q0) edge [loop left] node[above,yshift=4pt]{\neg}()
	(q0) edge[bend left] node{\pos}   (q1)
	(q1) edge[bend left] node{\pos}   (q2)
	(q2) edge[bend left] node{\pos}   (q3)
	(q3) edge[bend left] node{\neg}   (q2)
	(q2) edge[bend left] node{\neg}  (q1)
	(q1) edge[bend left] node{\neg}  (q0)
	(q3) edge [loop right]  node[above,yshift=4pt] {\pos}()
	;
	
	\node [left = \vardist of q0] {$ x_2 $};
	
	\node (q0) [state, below = 40pt of q0,fill=black,text=white] {\ex};
	\node (q1) [state, right = of q0] {\ex};
	\node (q2) [state, right = of q1] {\keep};
	\node (q3) [state, right = of q2] {\keep};
	
	\path [-stealth, thick]
	(q0) edge [loop left] node[above,yshift=4pt]{\neg}()
	(q0) edge[bend left] node{\pos}   (q1)
	(q1) edge[bend left] node{\pos}   (q2)
	(q2) edge[bend left] node{\pos}   (q3)
	(q3) edge[bend left] node{\neg}   (q2)
	(q2) edge[bend left] node{\neg}  (q1)
	(q1) edge[bend left] node{\neg}  (q0)
	(q3) edge [loop right]  node[above,yshift=4pt] {\pos}()
	;
	
	\node [left = \vardist of q0] {$ \overline{x_1}  $};
	
	\node (q0) [state, below = 40pt of q0] {\ex};
	\node (q1) [state, right = of q0] {\ex};
	\node (q2) [state, right = of q1,fill=black,text=white] {\keep};
	\node (q3) [state, right = of q2] {\keep};
	
	\path [-stealth, thick]
	(q0) edge [loop left] node[above,yshift=4pt]{\neg}()
	(q0) edge[bend left] node{\pos}   (q1)
	(q1) edge[bend left] node{\pos}   (q2)
	(q2) edge[bend left] node{\pos}   (q3)
	(q3) edge[bend left] node{\neg}   (q2)
	(q2) edge[bend left] node{\neg}  (q1)
	(q1) edge[bend left] node{\neg}  (q0)
	(q3) edge [loop right]  node[above,yshift=4pt] {\pos}()
	;
	
	\node [left = \vardist of q0] {$ \overline{x_2} $};
	
	\def\x{4.25}
	\draw[dotted] (\x,-7.)--(\x,.7);
	\draw[draw=black] (-1.55,1.) rectangle ++(11.6,-8);
	\draw[draw=black,fill=black,opacity=0.2] (-1.41,.92) rectangle ++(11.34,-1.8);
\end{tikzpicture}
\caption{4 TAs, each of them vote for the inclusion or exclusion of their associated variables. 
Black states indicate the final position in each TA after training.
	The recognised pattern is $ x_1 \overline{x_2}  $ as in function $ C^-(1) $ in Example~\ref{ex:1}.
	A single TA  voting for the inclusion of $ x_1 $ is depicted in the shaded area.
}
\label{f:tsautomaton}
\end{figure}

\begin{algorithm}[tb]
\begin{algorithmic}[1]
	\STATE {\bfseries Input:} Training data $\dataset$, features $\variables$, num. of monomials $\clauses$,  
	num. of states $\states$, margin $\target$, specificity $\precision$  
	\STATE {\bfseries Output:} the monomial functions $C^+,C^-$
	\STATE $\Amc^+,\Amc^-\leftarrow$ CreateTAs($\clauses,\states/2,\variables$) \label{a:tsm:3}
	\REPEAT
	\STATE $(\Imc,y) \leftarrow$ GetTrainingExample($\dataset$)\label{a:tsm:5}
	\STATE $C^+,C^- \leftarrow$ GetMonomials($\Amc^+,\Amc^-$)
	\STATE $\voting\leftarrow {\sf clip}(\sum^{n/2}_{j=1} C^-(j)-\sum^{n/2}_{j=1} C^+(j),-\target,\target)$  \label{a:tsm:7}
	\FOR{$j\leftarrow 1, \ldots, n/2$}
	\IF{$y=1$}
	\IF{rand() $\leq(\target - \voting)/(2 \target) $}
	\STATE $\Amc^+\leftarrow$ GenerateTypeIFeedback($\Imc,\Amc^+(j),\precision$)
	\STATE $\Amc^-\leftarrow$ GenerateTypeIIFeedback($\Imc,\Amc^-(j)$) 
	\ENDIF
	\ELSE
	\IF{rand() $\leq(\target + \voting)/(2 \target) $}
	\STATE $\Amc^+\leftarrow$ GenerateTypeIIFeedback($\Imc,\Amc^+(j)$)
	\STATE $\Amc^-\leftarrow$ GenerateTypeIFeedback($\Imc,\Amc^-(j),\precision$) 
	\ENDIF
	\ENDIF
	\ENDFOR
	\UNTIL{StopCriteria($\dataset, C^+,C^-$)}
	\STATE {\bfseries return} $C^+,C^-$ \label{ln:lastcase}
	
\end{algorithmic}
\caption{Tsetlin Machine}
\label{a:tsm}
\end{algorithm}
\section{Verifying Properties of TsMs} 
\label{sec:lr}
\label{sec:lr}

We first present an exact encoding of a TsM into propositional logic
using sequential counters (Section~\ref{sec:satencoding}). 
We then use this encoding to define the properties \emph{adversarial robustness}, 
\emph{equivalence}, and \emph{similarity} of TsMs. 

\subsection{SAT Encoding}
\label{sec:satencoding}

	
The encoding of TsMs into a propositional logic
uses the translation of sequential counters into a logic formula~\cite{10.1007/11564751_73}. 
Consider a cardinality constraint: 
$\sum^\ell_{i=1} l_i \geq K,$ where
$l_i \in \{0, 1\}$ 
and $K\in\mathbb{N}$. 
Then the cardinality constraint is 
encoded into   a formula as follows: 
\begin{equation}
	\begin{split}\label{eq:sc} 
		&(l_1 \leftrightarrow r_{1,1})\wedge (\neg r_{1,j}) \wedge
		(r_{i,1}\leftrightarrow (l_i\vee r_{i-1,1})) \wedge \\
		&(r_{i,j} \leftrightarrow ((l_i\wedge r_{i-1,j-1})\vee r_{i-1,j}))  
	\end{split} 
\end{equation}

where $i \in \{2,\ldots,\ell\}$ and $j \in \{2,\ldots,K\}$.
Variables of the form $ r_{i,j} $ are (fresh) auxiliary variables.
Intuitively, this encoding computes \emph{cumulative sums}. 
The variable 
$r_{\ell,K}$ is `true' iff
$\sum^\ell_{i=1} l_i$  
is greater than or equal to $K$. 
%
For a cardinality constraint $c:=\sum^\ell_{i=1} l_i \geq K$, 
we denote by $ \encsqc(c) $ the sequential counter encoding of $c$ into a formula  (Eq.~\ref{eq:sc}). 

To express the TsM formula $\Mmc=\sum^{n/2}_{j=1} C^-(j) \geq \sum^{n/2}_{j=1} C^+(j) $ 
in propositional logic, 
we define  $ \encts(\Mmc) $ as:
\begin{align}
	&\encts(\Mmc):=\bigwedge_{k\in\{+,-\}} \bigwedge_{j=1}^{n/2} ( \newvar{k}{j} \leftrightarrow C^k(j)^\dagger) \label{encoding:1} \\
	& \land\bigwedge_{k\in\{+,-\}}  \encsqc(\sum^{n/2}_{j=1}  \newvar{k}{j} \geq n/2) \label{encoding:2}  \\
	& \land \bigwedge_{j=1}^{n/2} ((   r^-_{n/2,j} \rightarrow r^+_{n/2,j} )\leftrightarrow o_j) \wedge ((\bigwedge_{j=1}^{n/2} o_j)\leftrightarrow o) \label{encoding:3}
\end{align}
where $ \newvar{+}{j}  $ 
and
$ \newvar{-}{j}  $ are fresh propositional variables, with  $ 1 \leq j\leq n/2 $,
and $r^k_{n/2,j}$ are the auxiliary variables in the encoding given by the
function $\encsqc(\cdot)$ for $k\in\{+,-\}$ (Eq.~\ref{eq:sc}).
We write $C^k(j)^\dagger$ for the result of converting a multiplication
of variables into a conjunction of variables and replacing 
each variable of the form  $\overline{x}$ by $\neg x$.
Eq.~\ref{encoding:3} expresses  that if $ j $  monomials in $ C^- $ are set to true,
then at least $ j $  monomials in $ C^+ $ are set to true. This information is stored in the variable $o$, which we call the \emph{output} variable of the encoding $ \encts(\Mmc) $ (we use additional auxiliary variables $o_j$ for practical purposes, in particular, to reduce the size of the formula when converting it into CNF, which is the format required by most SAT solvers).

\begin{example}\label{ex:2}\upshape
	Consider the TsM formula \Mmc in Example~\ref{ex:1}. Then, $\encts(\Mmc)$ is
	\begin{align*}
		& (v_{-,1}\leftrightarrow ( x_1\wedge \neg x_2))\wedge (v_{-,2}\leftrightarrow (\neg x_1\wedge x_2))
		& \wedge \\
		& (v_{+,1}\leftrightarrow ( x_1\wedge x_2))\wedge (v_{+,2}\leftrightarrow (\neg x_1\wedge \neg x_2)) & \wedge  \\ 
		& \bigwedge_{i\in\{+,-\}} ((v_{i,1} \leftrightarrow r^i_{1,1})\wedge (\neg r^i_{1,2})\wedge
		(r^i_{2,1}\leftrightarrow (v_{i,2}\vee r^i_{1,1}))& \wedge \\
		& (r^i_{2,2}\leftrightarrow ((v_{i,2}\wedge r^i_{1,1})\vee r^i_{1,2})))  \wedge  ((r^-_{2,1}\rightarrow r^+_{2,1})\leftrightarrow o_1) &\wedge \\
		&  ((r^-_{2,2}\rightarrow r^+_{2,2})\leftrightarrow o_2)\wedge ((o_1\wedge o_2)\leftrightarrow o). & 
	\end{align*}
\end{example}

We are now ready to state Theorem~\ref{thm:sat}, which essentially follows from the definition of $\encts(\Mmc)$ and the correctness 
of the sequential counter encoding.

\begin{restatable}{theorem}{theoremone}\label{thm:sat} 
	For a TsM formula \Mmc and an interpretation $\Imc$, 
	$\Mmc(\Imc)=1$ iff $\encts(\Mmc)_{[\vec{x}\rightarrow {\Imc}]}\wedge o$ is satisfiable.
\end{restatable}

\begin{proof}
	($ \Rightarrow $) $\Mmc(\Imc)=1$ implies that the number 
	of satisfied monomials in $ C^- $ by \Imc is at least equal to the number 
	of satisfied monomials in $ C^+ $ by \Imc.
	We show that we can find an interpretation $ \Jmc  $ that satisfies
	$ \encts(\Mmc)_{[\vec{x}\rightarrow {\Imc}]} \wedge o$. Initially we define $ \Jmc := \Imc $.
	According to Eq.~\ref{encoding:1} and for $k\in\{+,-\}$, we constraint \Jmc
	so that $ \Jmc(\newvar{k}{j})=1 $
	iff $ C^k(j) $ is satisfied by \Imc. At this point, \Jmc satisfies 
	the part of the encoding $ \encts(\Mmc)_{[\vec{x}\rightarrow {\Imc}]} $ in Eq.~\ref{encoding:1}.
	Then, for  $k\in\{+,-\}$,
	we count how many variables $\newvar{k}{j}$
	are set to true by \Jmc, where $ 1 \leq j \leq n/2 $. Let $ n_{k} $ be such number. 
	We set $ \Jmc( r^k_{n/2,j})=1 $ 
	iff $ j\leq n_{k} $ (the remaining variables of the form $r^k_{m,j}$ are 
	set to true if the sum of the first $m$ variables  $\newvar{k}{j}$ is at least $j$). 
	In this way, we satisfy Eq.~\ref{encoding:2} that corresponds to the second part of the encoding.
	Finally, Eq.~\ref{encoding:3} is already satisfied by the just built \Jmc because, since $\Mmc(\Imc)=1$, for every $ 1\leq j \leq n/2 $, 
	the rule $  r^-_{n/2,j} \rightarrow r^+_{n/2,j} $ is satisfied, and so, adding $o$ as a conjunct in $\encts(\Mmc)_{[\vec{x}\rightarrow {\Imc}]}\wedge o$ yields a satiafiable formula.
	%
	($ \Leftarrow $) Let \Jmc be an interpretation that satisfies $ \encts(\Mmc)_{[\vec{x}\rightarrow {\Imc}]} \wedge o$. 
	By Eq.~\ref{encoding:3},
	for all $ 1 \leq j \leq n/2 $, if 
	$r^-_{n/2,j}$ is true then $r^+_{n/2,j}$ is true.
	This means that, in Eq.~\ref{encoding:2}, if the sum of the 
	$n/2$ variables  $\newvar{-}{j}$ is at least $j$
	then this is so for the sum of the  $n/2$ variables  $\newvar{+}{j}$.
	By Eq.~\ref{encoding:1} and the definition of   $C^k(j)^\dagger$, this can only be if 
	$ \Mmc(\Imc)=1 $.
\end{proof}
\subsection{Robustness}
\label{subsec:robustness}

We use the encoding of TsMs   in Subsection~\ref{sec:satencoding} to 
define and show correctness of adversarial robustness for TsMs.


\begin{definition}[Adversarial   Robustness] \label{p1}
	Let $n$ be the dimension of the input of a TsM and let \Mmc be its  formula. 
	Such TsM is called \emph{$\epsilon$-robust} for an interpretation $\Imc$ 
 if  there is no
	interpretation $\Jmc$ 
	such that 
	$ \hamming_n(\Imc,\Jmc) \leq \epsilon $ and
	$\Mmc(\Imc )\neq \Mmc(\Jmc )$. 
\end{definition}

To check if a TsM is $\epsilon$-robust (Definition~\ref{p1}), we need 
to ensure that there are no $\epsilon$ bit flips of a given input vector 
$ \Imc $ such that the classification on $ \Imc $ of the TsM changes. 
We can explore the search space with the help of a SAT solver. 
To check for $\epsilon$-robustness, we call a SAT solver with the following formula 
$ \notrob(\Mmc,\Imc,\epsilon)  $ as input, where \Mmc is the TsM we would like to check robustness and
$ \Imc $ is an input for \Mmc. 
The formula is satisfied if the TsM is \emph{not robust}.
That is,
if we are able to find a combination of at most $\epsilon$
bit flips to apply to the input 
\Imc
such that the classification of the TsM
changes.

We can check with a SAT solver if 
there is no assignment that satisfy  $ \notrob(\Mmc,\Imc,\epsilon)  $, which means that 
the property of the TsM \Mmc  being adversarially $\epsilon$-robust is satisfied.
In the following, we write $\Mmc_{\vec{x}}$ to make explicit that
the $n$ variables $\vec{x}$ in \Mmc 
(the ones used to build the functions $C^k$
in 
$\encts(\Mmc)$  (Eqs.~\ref{encoding:1}-\ref{encoding:3}))
are those $x_j$ in Eq.~\ref{rob:xor}.
%
%
%
%
%
%
\begin{align}
	\notrob(\Mmc_{\vec{x}},\Imc,\epsilon ) := \;    \neg \encsqc(\sum_{j=1}^{n} l_j \geq \epsilon +1 ) \ \land 
	\label{rob:1} \\
	\bigwedge_{j=1}^{n} ( x_j \leftrightarrow (( \Imc[j] \lor l_j ) \land (\neg \Imc[j] \lor \neg l_j ))) \ \land \label{rob:xor} \\
 \encts(\Mmc_{\vec{x}}) \land  (\Mmc_{\vec{x}}(\Imc)\leftrightarrow \neg o) 
\end{align}

In Eq.~\ref{rob:1} we create new variables $ l_j $  with $ 1\leq j \leq n $
and we specify with a sequential counter encoding that at most $ \epsilon$ of such variables should have 
a `true' truth value. Semantically, an $ l_j $ is set to true if 
the truth value of $\Imc[i]$
should be flipped.
In Eq.~\ref{rob:xor}, we force the variables $ x_j $ with $ 1\leq j \leq n $,
that are used to define monomials in the TsM, to be the truth value of the flipped variable at position $ i $ such that 
$ x_j = \Imc[j] \oplus l_j $ (where $\oplus$ is the XOR operator).
That is, $ x_i $ is true if $ \Imc[i] $ is true and the bit should not be flipped  or if 
$ \Imc[i] $ is false and the value should be flipped.
In the other cases, the final value of the flipped variable is set to false. 
Finally, 
we add the constraint that the output of the TsM 
with the modified input (given by valuations of the variables $\vec{x}$) differs
from the label of the original input.
This means that the formula is \emph{satisfiable} if the TsM is
\emph{not $\epsilon$-robust} for the input \Imc.


We show in Example~\ref{ex:3} the not robust check 
for the TsM in Example~\ref{ex:1} and an interpretation classified as negative.

\begin{example} \label{ex:3}
	Let \Mmc be the TsM in Example~\ref{ex:1} and let $ \epsilon =1 $.
	We can check for the robustness of \Mmc with input the
	vector representation of an interpretation \Imc
	with $\Imc[1]=0$ and $\Imc[2]=0$
	(assume the dimension $n$ of the input of the TsM is $2 $) and $  \Mmc(\Imc)=0 $
	as follows:
	\begin{gather}
		\begin{align*}
			\neg & ((l_1\leftrightarrow t_{1,1}) \land \neg t_{1,2} \land (t_{2,1} \leftrightarrow (l_2 \lor t_{1,1})) &\land \\
			&(t_{2,2} \leftrightarrow ((l_2 \land t_{1,1}) \lor t_{1,2} ) )  &\land \\ 
			& 	(x_1 \leftrightarrow ((0\lor l_1) \land (1 \lor \neg l_1))) &\land \\
			& (x_2 \leftrightarrow ((0 \lor l_2) \land (1 \lor \neg l_2)))
			&\land \\
			& \encts(\Mmc_{\vec{x}}) \land  (0\leftrightarrow \neg o). 
		\end{align*}
	\end{gather}	
\end{example}

\begin{restatable}{theorem}{theoremtwo}\label{thm:robust}
	A TsM \Mmc is $ \epsilon$-robust for $ \Imc $ 
 iff
	$ \notrob(\Mmc,\Imc,\epsilon) $ is not satisfiable.
\end{restatable}

In a similar way, one can also check for a stronger property,
denoted universal adversarial  robustness. 
This property holds for a set of   inputs if a TsM
is robust to all adversarial perturbations (up to some threshold value) in a number of elements of this set.
More specifically, if $ S $ is the set of   inputs considered, we would like to check
if the TsM \Mmc  classifies differently only an $ \eta $-fraction of 
perturbed elements in $ S $.

\begin{definition}[Universal Adversarial  Robustness] 
Let $n$ be the dimension of the input of a TsM with formula \Mmc.
	Then, \Mmc is  $(\epsilon,\eta)$-robust for a set of 
	  classified inputs $S$ if it is 
   $\epsilon$-robust for $\eta|S|$ or more
   classified inputs in $S$.
\end{definition}

To check for this property, we check if a SAT solver
does not find any assignment for the formula:
\begin{align}
	&{\sf{UniRob}}(\Mmc,S,\epsilon,\eta) := \; \encsqc( \sum_{i=1}^{|S|} c_i \geq \lfloor\eta |S|\rfloor) \\
	& \bigwedge_{(\Imc_i,\Mmc(\Imc_i)) \in S} (\notrob(\Mmc,\Imc_i,\epsilon) \leftrightarrow \neg c_i)  
\label{unirob:2}
\end{align}

In Eq.~9, each $ c_i $ is a new variable that is set to true iff the $\epsilon$-robustness check with a specific   example $ \Imc_i $ has not passed. 
Then, in Eq.~\ref{unirob:2}, we check whether 
the number of examples in which the non-robustness 
test failed 
passes a proportion of the set $S$
based on $\eta$.

\begin{restatable}{theorem}{theoremthree}    
\label{thm:unirobust}
	A TsM \Mmc is $ (\epsilon,\eta )$-robust for a set of   inputs $ S $ iff
	$ {\sf{UniRob}}(\Mmc,S,\epsilon,\eta) $ is  satisfiable.
\end{restatable}

\subsection{Equivalence and Similarity}\label{sec:equi-sim}
Two TsMs 
are equivalent if they output the same class 
given the same input (Definition~\ref{def:equiv}).

\begin{definition}\label{def:equiv}
	Two TsMs $\Mmc_1$ and $\Mmc_2$ 
	are equivalent
	if, for all interpretations \Imc (as an input vector),
	$ \Mmc_1(\Imc) = \Mmc_2(\Imc)  $.
	
\end{definition}

We can use the encoding $ \encts(\Mmc_i) $ for $ 1\leq i\leq2 $
presented in Section~\ref{sec:satencoding}
to search for an input \Imc that is classified differently by them.
We assume a deterministic 
procedure for generating the 
 variables
in the encoding $ \encts(\Mmc_i) $, which guarantees
that, if the dimension of the input vectors of the TsMs with formulas $\Mmc_1$ and $\Mmc_2$ are the same,
then 
$ \encts(\Mmc_1) $ and $ \encts(\Mmc_2) $ are formulated
using the same input variables $\vec{x}$. 



\begin{restatable}{theorem}{theoremfour}   
\label{thm:equiv}
	Two TsMs $\Mmc_1$ and $\Mmc_2$ 
	are  equivalent iff $(\encts(\Mmc_1) \wedge o^1)\leftrightarrow (\encts(\Mmc_2)\wedge o^2)$
	holds, where $o^1,o^2$ are the output variables of $\encts(\Mmc_1)$ and $\encts(\Mmc_2)$, respectively.
\end{restatable}

One can imagine that a complete equivalence between machine learning models is difficult to achieve. It is therefore also interesting to consider the 
case in which TsMs are not equivalent but 
similar.
We define similarity w.r.t. a set of  inputs 
(and small perturbations) as follows.

\begin{definition}[Similarity and Universal Similarity]\label{def:sim}
	Let $n$ be the dimension of the input of TsMs with 
	formulas $\Mmc_1$ and $\Mmc_2$.
	Then, 
 $\Mmc_1$ and $\Mmc_2$ are $\epsilon$-similar for an interpretation \Imc if
 they   give the same classification result for \Imc
 and all interpretations \Jmc such that 
 $ \hamming_n(\Imc,\Jmc) \leq \epsilon $.
  Moreover,
 $\Mmc_1$ and $\Mmc_2$ are $(\epsilon,\eta)$-similar for a set $S$ of 	inputs if they are 
 $\epsilon$-similar for $\eta |S|$ or more inputs in $S$.
\end{definition}

To check for similarity on an input \Imc, we check if a SAT solver
does not find any assignment for the formula:
\begin{align*}	\notsim(\Mmc_{1,\vec{x}},\Mmc_{2,\vec{x}},\Imc,\epsilon) := \;    \neg \encsqc(\sum_{j=1}^{n} l_j \geq \epsilon +1 ) \ \land 
 \\
	\bigwedge_{j=1}^{n} ( x_j \leftrightarrow (( \Imc[j] \lor l_j ) \land (\neg \Imc[j] \lor \neg l_j ))) \ \land 
 \\
\encts(\Mmc_{1,\vec{x}}) \land \encts(\Mmc_{2,\vec{x}}) \land (o^1\leftrightarrow \neg o^2) 
\end{align*}

To check for universal similarity, we check if a SAT solver
does not find any assignment for the formula:
\begin{align*}
	&{\sf{UniSim}}(\Mmc_{1,\vec{x}},\Mmc_{2,\vec{x}},S,\epsilon,\eta ) := \; \encsqc( \sum_{i=1}^{|S|} c_i \geq \lfloor\eta |S|\rfloor)\\
	&  \bigwedge_{\Imc_i \in S} (\notsim(\Mmc_{1,\vec{x}},\Mmc_{2,\vec{x}},\Imc_i,\epsilon) \leftrightarrow \neg c_i)  
\end{align*}
The intuition for 
${\sf{UniSim}}(\Mmc_1,\Mmc_2,S,\epsilon,\eta )$
is similar to the intuition for 
$	{\sf{UniRob}}(\Mmc,S,\epsilon,\eta)$
except that here we check for two TsMs (with the same dimension).

\begin{restatable}{theorem}{theoremfive}  
\label{thm:unirobust}
	 TsMs $\Mmc_1$ and $\Mmc_2$ are 
      (1) $ \epsilon$-similar for an
      interpretation \Imc iff
	$ {\sf{NotSim}}(\Mmc_1,\Mmc_2,\Imc,\epsilon) $ is unsatisfiable; and (2)
 $ (\epsilon,\eta )$-similar for a set of   inputs $ S $ iff
	$ {\sf{UniSim}}(\Mmc_1,\Mmc_2,S,\epsilon,\eta) $ is satisfiable.
\end{restatable}

\section{Experiments and Results}\label{sec:exp}
We present the results of our experiments to verify properties of TsMs on the classical MNIST and IMDB datasets. 
In particular, we employ
the 
formulas presented in Section~\ref{sec:lr} for checking adversarial robustness, equivalence, and similarity for image classification and sentiment analysis. We also perform some tests to compare TsMs and BNNs.
To perform  our experiments, we wrote the code in Python~3.8 and we used 
 the Glucose SAT solver~\cite{simon}, 
 which is the same SAT solver employed by Narodytska et al. in her mentioned work\footnote{We attempted to use the SAT solver by Jia et al. but we got out-of-memory errors that ultimately prevented us from using the SAT solver presented in their 
 work~\cite{DBLP:conf/nips/JiaR20}.}. 
We now describe the datasets 
used for training  in more detail. 
\begin{itemize}
	\item The \emph{\textbf{MNIST}} dataset consists of $60,000$ gray-scale 
	$28\times 28$ images for training for the task of hand written single 
	digits recognition and $10,000$ for testing. 
Images are binarized using an adaptive 
Gaussian thresholding procedure as proposed in \cite{granmo2019convtsetlin}. 
 
	\item The \emph{\textbf{IMDB}} dataset is a sentiment classification dataset 
	consisting of  $50,000$ movie reviews, where $25,000$ of the samples are used as training data and the other
half as the testing set. The text is binarized using bags-of-words including   the $5,000$ most
frequent words.
\end{itemize}
%
%
We run all experiments on  
an AMD Ryzen 7 5800X CPU at 3.80GHz with 8 logical
cores, Nvidia GeForce GTX 1070, and 32GB RAM.
The code is available at \url{https://github.com/bimalb58/Logical-Tsetlin-Machine-Robustness}.



\subsection{Training TsMs}
\label{subsec:training} 
\begin{table}[t]
 \caption{TsMs models trained on MNIST. }
	\label{tab:exeprimentMNISTmodels}
	\begin{center}
		\begin{tabular}{lllllll}
			\toprule
			& \textbf{N} & \textbf{T} & \textbf{s} & \textbf{train acc} & \textbf{test acc} &\textbf{train time}\\ \midrule
			1 & 500       & 25         & 10         & 99.20  \%                 & 97.41 \%         & 1880.45s         \\
			2 & 1000       & 25         & 10         & 99.94 \%                  & 98.22 \%             & 3447.27s     \\
			3 & 2000       & 50         & 10         & 99.97 \%                  & 98.25 \%     & 6915.48s          
		\end{tabular}
	\end{center}
\end{table}
We train three  multiclass TsMs 
on the MNIST dataset varying the number of monomials $N$. The size of monomials affects the complexity and size of the model (and, therefore, of the formula used to check for robustness).
We set
the limit of the maximum size of monomials to $2000$.
Table~\ref{tab:exeprimentMNISTmodels} presents the results of the accuracy of the training, accuracy of testing, and training
time for TsMs trained on the MNIST dataset. 
The test accuracy
varies approximately less than 1\% for all the models for image classification. 

Table~\ref{tab:exeprimentIMDBmodels} contains the results of the  binary classification TsM models for sentiment analysis on the IMDB dataset. For the purpose of equivalence and similarity verification,
presented later, two of the models are trained using the same
hyperparameters. The first model is trained using a different value for the specificity  $s$. This hyperparameter sets the probability of TsM to memorize a literal.
The higher the value of $s$, the more literals are included in each monomial of the model.
The last column of the table presents the average number of literals per monomial. The
model having fewer literals reports higher accuracy on the testing data. The choice of the
parameters for training the models is based on a recent paper \cite{DBLP:conf/ijcai/Yadav0GG22} which
presents a correlation between the $s$ hyperparameter and model robustness (not tested precisely with a SAT solver, as we do in this work). 
%
All models in Tables~\ref{tab:exeprimentMNISTmodels} and ~\ref{tab:exeprimentIMDBmodels} were trained with $400$ iterations of Algorithm~\ref{a:tsm}.

\begin{table}[t]
\caption{TsMs models trained on IMDB. Models 2 and 3 have the same hyperparameters and  were used for  verifying equivalence and similarity (Table~\ref{tab:sim}).}
	\label{tab:exeprimentIMDBmodels}
	\begin{center}
		\begin{tabular}{lllllll}
			\toprule
			& \textbf{N} & \textbf{T} &\textbf{s} & \textbf{train acc} & \textbf{test acc} & \textbf{avg. lit}\\ \midrule
			1    & 1000 &   1280   & 20         & 78.52  \%                 & 76.65 \%     &     225        \\
			2    & 1000 &   1280    & 2        & 83.91 \%                  & 82.53 \%           &   173    \\
   3    & 1000 &    1280   & 2        & 84.20 \%                  & 82.53 \%           &  175     \\
		\end{tabular}
	\end{center}
\end{table}
\subsection{Adversarial Robustness}
\label{subsec:adv-rob}

\begin{figure}
    \centering
    \includegraphics[width=\linewidth]{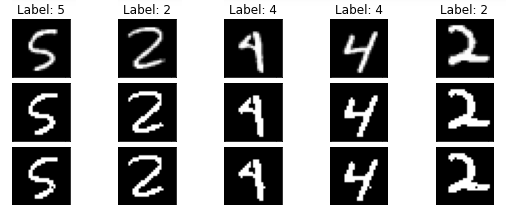}
    \caption{Original (the top row), binarized (the middle row), and perturbed with $\epsilon = 1$
(the bottom row) sample MNIST images.}
    \label{fig:my_label}
\end{figure}
We present our results for adversarial robustness
using the MNIST and IMDB datasets.
To test adversarial robustness on an input \Imc, we experiment with three different $\epsilon$ perturbation values by varying $\epsilon \in\{1, 3, 5\}$.
These values correspond to the
maximum number of bit-flips
applied on \Imc when checking robustness using the SAT solver. 
The solver might take exponential time in the worst cases, so
we impose a timeout of $300$ seconds on the solver for each test instance. This
setup is the same for all experiments  in this section. The range of $\epsilon$ and the timeout
are chosen in this way to facilitate the comparison with BNNs~\cite{DBLP:conf/aaai/NarodytskaKRSW18}  at the end of this section. 

For  multi-class  problems,  such as MNIST classification,
the current state of the art of TsMs~\cite{tsetlincuda}
creates one TsM \Mmc for each target class during the pre-processing phase. 
This means that the final trained model consists of a sorted team of TsMs 
where each member $ \Mmc_i $ will individually compute the sum (Eq.~\ref{eq:polynomial})
that votes for class $ i $.
Given  an interpretation \Imc (in fact its vector representation) as input, the TsM team will output the class associated with the 
TsM that has the highest sum, as explained in Section~2.2.

As the MNIST dataset contains $ 10 $   classes, we conduct 
the robustness experiment on each single TsM belonging to the trained TsM team.
We test the adversarial robustness on $20$ randomly selected MNIST   images per class, summing up to $200$ images.
Table~\ref{tab:robmnist} presents the robustness results for the three MNIST models from Table~\ref{tab:exeprimentMNISTmodels}. Columns
show the maximum number of bit-flips $ \epsilon $, the number of instances solved by the SAT solver, the amount of them being $\epsilon$-robust,
and the average time in seconds to solve these instances. ``Solved'' in this context means that
the SAT solver either finds a valid perturbation leading to the change in the classification of the
input image or   that   the model is $\epsilon$-robust
on this image. Figure~\ref{fig:my_label}  illustrates examples of  original images and  perturbed
images with $\epsilon = 1$ (that is, with $1$-pixel flip) which led to a change in the classification. 
\begin{table*}[t]
\caption{Robustness on the MNIST dataset.}
	\label{tab:robmnist}
\centering
\begin{tabular}{|c|c|c|c|c|c|c|c|c|c|}
	\hline
	\multicolumn{10}{|c|}{200 test instances} \\
 	\hline
  \hline
 &	 \multicolumn{3}{|c|}{model 1} & \multicolumn{3}{c|}{model 2}  & \multicolumn{3}{c|}{model 3} \\
	\hline
$\epsilon$ & solved	& $\epsilon$-robust & time (sec)  & solved	& $\epsilon$-robust & time (sec) &  solved	& $\epsilon$-robust & time (sec) \\
\hline
 1 & {200} & {171} & {0.92} & {200} & {184} & {2.81} & {200} & {184} & {13.85}	
  \\
   	\hline
 3 & 162 & 45 & 26.11  & 57 & 29 & 105.85  & 7 & 5 & 258.36 
\\
	\hline
5 & 124 & 0 & 70.61  & 12 & 0 & 268.54  & 0 & 0 & $\geq$ 300   \\
	\hline
\end{tabular}
\end{table*}
For $\epsilon = 1$, all instances were  solved withing the timeout of $300$ seconds and Models 2 and~3 (Table~\ref{tab:exeprimentMNISTmodels}), which are trained with more monomials, demonstrated
greater resistance to input perturbations.
However, increasing $\epsilon$ yields a more complex robustness encoding and solving more complex formulas, built from larger models such as Models 2 and 3, requires more time from the SAT solver. The solver was unable to complete any of the
instances for $\epsilon = 5$ for Model~3 within $300$ seconds (the timeout value). 

Table~\ref{tab:robimdb} presents the results of adversarial robustness for models trained on the IMDB
dataset. The goal of this experiment
is to study the effect of changing the
specificity value of the TsMs for robustness
verification.
The experiment is run on $100$ test instances that were correctly classified by both
models at the input without perturbation, that is, $\epsilon = 0$. The models were trained to take a binary vector of $5000$ features as an input. Robustness test
for $\epsilon = 1$ creates $5000$ possible combinations of bit flips, that is, including or not including
a word in the sentence (this is much more computationally challenging than the $28\times 28=784$ possible combinations of bit flips for $\epsilon=1$ in MNIST). For 
$\epsilon=5$, there are more than 120 billion
possible combinations. Model 2 (and 3) with a low $s$ value, that stimulates negated reasoning, 
resulted in greater resistance to adversarial inputs, 
as expected~\cite{DBLP:conf/ijcai/Yadav0GG22}. 
Due to the fact
that Model 2 contains fewer literals per monomial, it was able to solve more instances within   300 sec., with lower average computation time when compared to Model~1. 



\begin{table}[t]
\caption{Robustness on the IMDB dataset.}
	\label{tab:robimdb}
\centering
\begin{tabular}{|c|c|c|c|c|c|c|}
	\hline
	\multicolumn{7}{|c|}{100 test instances} \\
 	\hline
  \hline
 &	 \multicolumn{3}{|c|}{model 1} & \multicolumn{3}{c|}{model 2}     \\
	\hline
$\epsilon$ & solved	& $\epsilon$-rob & time (s)   & solved	& $\epsilon$-rob & time (s)  \\
\hline
 1 & 98 & 36  & 42.73 & 100 & 56 & 10.92
  \\
   	\hline
 3 & 54 & 0 & 233.59  & 83 & 17 & 69.72
\\
	\hline
5 & 49 & 0 &  217.95 & 83 & 8 &  64.50\\
	\hline
\end{tabular}
\end{table}
 

 
\subsection{Equivalence and Similarity}\label{subsec:equivalence}

We run an equivalence test on the two models trained with the IMDB dataset using the same hyperparameter setup, that is, Model 2 and Model 3 from Table~\ref{tab:exeprimentIMDBmodels}. Both models achieved the same
accuracy on the test set. This test aims to identify if two models trained on the same dataset and having the same accuracy
would always give the same classification results for any input. The equivalence test, which verifies if their corresponding formulas are logically equivalent, showed
them to not  be fully equivalent in this strong sense.
We then considered similarity, which is easier to achieve as it only requires that the classification matches in a number of instances, 
for a particular dataset. 
Table~\ref{tab:sim} presents  the similarity experiment on a test set with $100$ instances. Columns show the number of instances
solved by the SAT solver and the number of $\epsilon$-similar instances. ``Solved' in this context
means that the SAT solver determines within 300 sec. whether two models return the same classification
  given an identical input (and perturbations of it quantified by  $\epsilon$).  The experiment is run with three different $\epsilon$
values. The results reported for $\epsilon  = 0$ show the similarity of two models in the input without
perturbations. Models 2 and 3 are more consistent in their output predictions. They
resulted in a single different classification for unperturbed input. Increasing the $\epsilon$ makes
the SAT solver time out on several instances. Even then, Models 2 and 3 are  $\epsilon$-similar on more instances for $\epsilon=1$.
\begin{table}[t]
\caption{Similarity check on the IMDB dataset.}
	\label{tab:sim}
\centering
\begin{tabular}{|c|c|c|c|c|}
	\hline
	\multicolumn{5}{|c|}{100 test instances} \\
 	\hline
  \hline
 &	 \multicolumn{2}{|c|}{model 1 \& model 2} &\multicolumn{2}{c|}{model 2 \& model 3 }     \\
	\hline
$\epsilon$ &solved	& $\epsilon$-similar  &   solved	& $\epsilon$-similar  \\
\hline
0 & 100 & 93   & 100 & 99  
  \\
   	\hline
 1 & 93 & 43   & 74 & 70     
\\
	\hline
3 & 59 & 41   & 60 & 40   \\
	\hline
\end{tabular}
\end{table}

 \subsection{Comparison with BNNs}\label{sec:tsmbnn}
In this section, we compare   
TsMs and BNNs w.r.t. accuracy
on the MNIST-c dataset and we discuss 
robustness in both cases on the MNIST-back-image dataset.

\paragraph{Accuracy Tests}
Table \ref{tab:d} shows the comparison of accuracy results on MNIST-c for the state-of-the-art TsM   for MNIST (with $N=8000$,  $ T = 6400$,  and $s = 5$.)
and a BNN using 
the same hyperparamaters as in \cite{DBLP:conf/nips/JiaR20}. 
Both the TsM and the BNN  are trained on the standard MNIST training data and tested on MNIST-c. 
Test images are sampled from MNIST-c as follows: with $p$ probability, 
a test image from the corrupted data is selected; 
one of the $15$ corruptions is selected with uniform probability, 
where $p=0$ means that the standard uncorrupted test data is used. 
In Table \ref{tab:d}, we can see that the TsM achieves competitive 
accuracy results with different values of $p$, indicating  good performance of TsMs in comparison 
with BNNs on the MNIST-c test data (note that when $p=0$ then
we have the classical MNIST).

\begin{table}[h]
 \caption{Accuracy comparison on the MNIST-c dataset}
	\label{tab:d}
	\centering
	\begin{tabular}{ccccccc}
		\toprule
		MNIST-c & $p=0$	& $p=0.25$ & $p=0.5$ & $p=0.75$ \\
		\midrule
		TsM	&	99.3 & 92.5   & 88.1 &  81.2   \\
		BNN &	97.46 & 90.1   &  83.5 &  77.2    \\
		\bottomrule
	\end{tabular}
\end{table}

\paragraph{Robustness}
We now consider the robustness results for BNN reported by Narodytska et al.
in their work to verify the properties of BNNs~\cite{DBLP:conf/aaai/NarodytskaKRSW18}, with the robustness results achieved by
TsMs. Unfortunately, their code is not available to the community, making it impossible
to rerun their code in the same conditions. 
We consider the results for $\epsilon$-robustness using the MNIST-back-image
dataset, presented in the main text of in their paper. 
To visualize the results under  similar conditions, we trained the TsM
model  using the MNIST-back-image dataset. Both the TsM and BNN models considered in this experiment
are scaled down for 
 robustness verification.
The accuracy of the TsM (with $N=1000$, $T=25$, and $s=10$)   after $400$ iterations was $81.56\%$ while the accuracy of the BNN model used by Narodytska et al. was $70\%$~\cite{DBLP:conf/aaai/NarodytskaKRSW18}.
  We 
 randomly selected $20$ images for each of the $10$ classes, which results in $200$ test instances.
\begin{table}[t]
\caption{Robustness on the MNIST-back-image dataset.}
	\label{tab:robimdbtwo}
\centering
\begin{tabular}{|c|c|c|c|c|c|c|}
	\hline
	\multicolumn{7}{|c|}{200 test instances} \\
 	\hline
  \hline
 &	 \multicolumn{3}{|c|}{TsM} & \multicolumn{3}{c|}{BNN }     \\
	\hline
$\epsilon$ &solved	& $\epsilon$-rob & \% $\epsilon$-rob   & solved	& $\epsilon$-rob & \% $\epsilon$-rob  \\
\hline
1 & 200 & 145 & 72.5\% & 191 & 138 & 72.25\%
  \\
   	\hline
 3 & 116 & 14 & 12.07\% & 107 & 20 & 18.69\%   
\\
	\hline
5 & 34 & 0 & 0\% & 104 & 3 & 2.88\% \\
	\hline
\end{tabular}
\end{table}
The experiment is run with three different perturbation values $\epsilon \in \{1, 3, 5\}$. It uses the
same SAT solver as the authors of the BNN verification paper, i.e., Glucose, and the same
timeout value of 300 seconds for each test instance. Table~\ref{tab:robimdbtwo} presents results for both
models. The hardware used in both experiments is not the same. 
However, the percentages of
$\epsilon$-robust instances in these two models are similar. This indicates that both models can
have approximate robustness performance 
on  MNIST-back-image perturbed inputs, however more tests are needed.

\section{Conclusion}\label{sec:conclusion}
%
We present an exact encoding of TsMs into propositional logic and we show
how to verify properties such as adversarial robustness
using a SAT solver, following an earlier approach for BNNs. 
We show the correctness 
of our encoding and present experimental results for adversarial robustness, equivalence, similarity. 
We then compare the accuracy  between
TsMs and BNNs,
using the MNIST-c dataset (designed for testing accuracy with corrupted instances), and discuss robustness in both cases using the MNIST-back-image dataset.
As future work, we plan to investigate optimizations of SAT solvers~\cite{DBLP:conf/nips/JiaR20}  for robustness verification. 
%

\section{Acknowledgements}
Ozaki is supported by the NFR projects 316022 and 322480.


\section{Learning with Tsetlin Machines}
In this section we provide more details about   TsMs~\cite{granmo2018tsetlin} and a running example for sentiment classification. 

\usetikzlibrary{arrows,automata}
\subsection{Tsetlin Machines: Feedback Types}
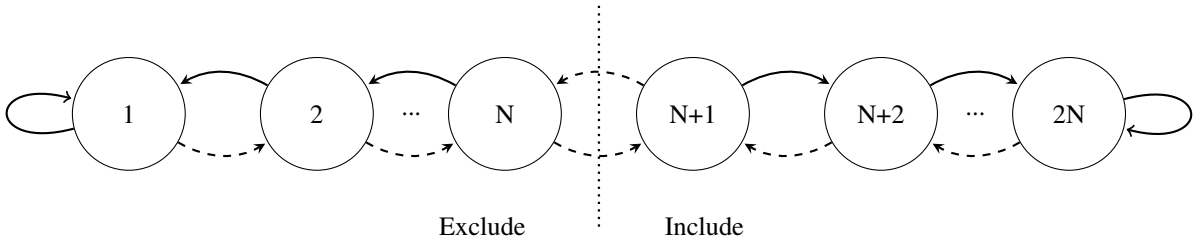
\begin{figure*}[h]
	\begin{tikzpicture}[shorten >=1pt,node distance=2.5cm,on grid,auto]
		
		\node (q0) [state, minimum size=1.5cm] at(0,0) {1};
		\node (q1) [state, right = of q0, minimum size=1.5cm] {2};
		\node (q2) [state, right = of q1, minimum size=1.5cm] {N};
		\node (q3) [state, right = of q2, minimum size=1.5cm] {N+1};
		\node (q4) [state, right = of q3, minimum size=1.5cm] {N+2};
		\node (q5) [state, right = of q4, minimum size=1.5cm] {2N};
		
		\path [-stealth, thick]
		(q0) edge [loop left] node[above,yshift=4pt]{}()
		(q1) edge[bend right] (q0)
		(q2) edge[bend right] (q1)
		(q3) edge[bend right][dashed] (q2)
		(q3) edge[bend left]  (q4)
		(q4) edge[bend left]  (q5)
		(q2) edge[bend right][dashed] (q3)
		(q4) edge[bend left][dashed]  (q3)
		(q5) edge[bend left][dashed]  (q4)
		(q1) edge[bend right][dashed] (q2)
		(q0) edge[bend right][dashed] (q1)
		(q5) edge [loop right]  node[above,yshift=4pt] {}();
		
		\draw[dotted, line width=0.3mm] (6.25,-1.5)--(6.25,1.5);
		
		\node[text width=6cm, anchor=west, right] at (7,-1.5)
		{Include};
		\node[text width=6cm, anchor=west, right] at (4,-1.5)    
		{Exclude};
		\node[text width=6cm, anchor=west, right] at (3.5,0)
		{...};  
		\node[text width=6cm, anchor=west, right] at (11,0)
		{...};              
	\end{tikzpicture}
	\caption{A Tsetlin Automaton (reward is represented on dashed arrows and penalty on simple arrows)}
	\label{ta}
\end{figure*}


The name ``Tsetlin Machine'' originates from the Tsetlin Automaton (Figure \ref{ta}), introduced by M. L. Tsetlin in 1961.
TsMs use hardware-near bitwise operators, thereby minimizing the memory usage and computation cost. Furthermore, the whole learning process and recognition is based on bit manipulation. Before starting the TsM training, all the inputs need to be binarized. There exist several techniques for binarizing the dataset depending on the training data format and the objective of the model. For image classification tasks, it is recommended to binarize the training data using an adaptive Gaussian thresholding procedure \cite{granmo2019convtsetlin}. Binarization will result in $1$ bit per pixel channel. Although TsMs have shown competitive accuracy results on simple grayscale image data sets such as MNIST or Fashion-MNIST, they still show quite poor accuracy for colored image datasets, for example, CIFAR-100, compared to other state-of-the-art models \cite{granmo2019convtsetlin}. Binarizing natural language is less harmful because many text vectorization techniques already produce binary vectors, e.g., bag-of-words. As such, one loses less information than when binarizing image inputs.

The basic  TsM for binary classification takes as input an  $s$-dimensional vector $X=[x_1, ..., x_s]$ 
and produces a classification output $\hat{y}$. The output $\hat{y}$ is one of two possible classes, $\hat{y}=1$ or $\hat{y}=0$, corresponding to \emph{true} and \emph{false}. 
We associate each position in the input vector to a variable $x_i$, with $1\leq i\leq s$,  
and form a  literal set $$L=\{x_1, ..., x_s, \neg{x_1}, ..., \neg{x_s} \}.$$ 
Trained TsMs are represented as formulas, using monomials, also called \emph{conjunctive clauses}~\cite{granmo2018tsetlin}, $\mathcal{C}^{+}$ and $\mathcal{C}^{-}$, consisting of literals from the  set $L$. The number of monomials is predefined by $n$, which is an input parameter to  Algorithm \ref{a:tsm}. Half of the monomials are assigned positive polarity and the other half are assigned negative polarity. Thus, the size of $\mathcal{C}^{k}$, where $k\in\{+,-\}$, is $n/2$. The monomial outputs are combined into a classification decision, 
as explained in the main text (see Equation~\ref{eq:polynomial}).

In simple words, the classification is based on majority voting where monomials with positive polarity are voting for \emph{true} classification and monomials with negative polarity for \emph{false} classification. 
This formula provides an interpretable explanation of the model, which is useful to check whether decisions are unfair, biased, or erroneous (see Example~1).

Each monomial $C^k(j)$ is composed from the set of literals $L$ whereby each of the literals is associated with its own TA. The automation decides whether to \emph{include} or \emph{exclude} a given literal from the literal set $L$ in the monomial $C^k(j)$. As illustrated in Figure \ref{ta}, the decision to \emph{include} or \emph{exclude} a literal is based on the function $\Gmc(\phi_a)$, where $\phi_a$, $a\in\{1, 2, ..., 2N\}$ is the current state of the TA. 
\begin{equation*}
	\Gmc(\phi_a) =
	\begin{cases}
		exclude, &  1 \leq a \leq N \\
		include, &  N+1 \leq a \leq 2N
	\end{cases}
\end{equation*}
$N$ is a hyperpartameter of the TsM  model. 
The training procedure is given by Algorithm \ref{a:tsm}. The state transition of each Tsetlin Automaton governs learning. The reward transitions are indicated with solid lines in Figure \ref{ta}, and the penalization transitions are indicated with dotted lines. Learning which literals to include in the monomials is based on reinforcement. Each monomial $C^k(j)$ is trained by receiving one of two types of feedback, depending on their polarity and desired output classification. Both feedbacks are  stochastic. The value computed by the function $\sf clip$ in Algorithm \ref{a:tsm}   bounds the sum between the interval $[-T, T]$ and it is used to guide the random selection for the feedback type. Monomials can be given either a Type I Feedback that generalizes by producing frequent patterns or a Type II Feedback that specializes and strictly regulates the patterns \cite{granmo2018tsetlin}.

\noindent
\textbf{Type I Feedback}   is given to the TA of monomials $C^k(j)$ when the output is equal to $k$ (that is, when they are correct). Table~\ref{tbl:feed1} contains the probabilities of receiving ``Reward'', ``Inaction'', or ``Penalty'' given the monomial polarity and literal value. Inaction feedback is a novel extension of the TA introduced by Granmo (2018).
Inaction is simply leaving the TsM untouched. 
The variable $s$ is a  hyperparameter fed to the learning algorithm that represents specificity. 
It controls how strongly the model prefers to include literals in the monomial. The greater $s$, the more literals are included. 

\textbf{Type II Feedback} is given to the TA of monomial $C^k(j)$ when the output is not equal to $k$. Type II feedback actions are visualized in Table \ref{tbl:feed2}. Type I and Type II Feedbacks aim together at minimizing the output error.

\begin{table}[t]
	\centering
	\begin{tabular}{|c|c|c|c|c|c|}
		\hline
		\multicolumn{1}{|c|}{\textbf{Action}} & \multicolumn{1}{|c|}{\textbf{Monomial}} & \multicolumn{2}{|c|}{1} & \multicolumn{2}{|c|}{0} \\
		\cline{2-6}
		\multicolumn{1}{|c|}{} & \multicolumn{1}{|c|}{\textbf{Literal}} & \multicolumn{1}{|c|}{1} & \multicolumn{1}{|c|}{0} & \multicolumn{1}{|c|}{1} & \multicolumn{1}{|c|}{0} \\
		\cline{1-6}
		\multicolumn{1}{|c|}{\emph{Incl.}} & \multicolumn{1}{|c|}{P(Reward)} & \multicolumn{1}{|c|}{$\dfrac{s-1}{s}$} & \multicolumn{1}{|c|}{-} & \multicolumn{1}{|c|}{0} & \multicolumn{1}{|c|}{0} \\ [2ex]
		\cline{2-6}
		\multicolumn{1}{|c|}{} & \multicolumn{1}{|c|}{P(Inaction)} & \multicolumn{1}{|c|}{$\dfrac{1}{s}$} & \multicolumn{1}{|c|}{-} & \multicolumn{1}{|c|}{$\dfrac{s-1}{s}$} & \multicolumn{1}{|c|}{$\dfrac{s-1}{s}$} \\ [2ex]
		\cline{2-6}
		\multicolumn{1}{|c|}{} & \multicolumn{1}{|c|}{P(Penalty)} & \multicolumn{1}{|c|}{0} & \multicolumn{1}{|c|}{-} & \multicolumn{1}{|c|}{$\dfrac{1}{s}$} & \multicolumn{1}{|c|}{$\dfrac{1}{s}$} \\ [2ex]
		\cline{1-6}
		\multicolumn{1}{|c|}{\emph{Excl.}} & \multicolumn{1}{|c|}{P(Reward)} & \multicolumn{1}{|c|}{0} & \multicolumn{1}{|c|}{$\dfrac{1}{s}$} & \multicolumn{1}{|c|}{$\dfrac{1}{s}$} & \multicolumn{1}{|c|}{$\dfrac{1}{s}$} \\ [2ex]
		\cline{2-6}
		\multicolumn{1}{|c|}{} & \multicolumn{1}{|c|}{P(Inaction)} & \multicolumn{1}{|c|}{$\dfrac{1}{s}$} & \multicolumn{1}{|c|}{$\dfrac{s-1}{s}$} & \multicolumn{1}{|c|}{$\dfrac{s-1}{s}$} & \multicolumn{1}{|c|}{$\dfrac{s-1}{s}$} \\ [2ex]
		\cline{2-6}
		\multicolumn{1}{|c|}{} & \multicolumn{1}{|c|}{P(Penalty)} & \multicolumn{1}{|c|}{$\dfrac{s-1}{s}$} & \multicolumn{1}{|c|}{0} & \multicolumn{1}{|c|}{0} & \multicolumn{1}{|c|}{0} \\ [2ex]
		\cline{1-6}
	\end{tabular}
	\caption{Type I Feedback}
	\label{tbl:feed1}
\end{table}


\begin{table}
	\centering
	\begin{tabular}{|c|c|c|c|c|c|}
		\hline
		\multicolumn{1}{|c|}{\textbf{Action}} & \multicolumn{1}{|c|}{\textbf{Monomial}} & \multicolumn{2}{|c|}{1} & \multicolumn{2}{|c|}{0} \\
		\cline{2-6}
		\multicolumn{1}{|c|}{} & \multicolumn{1}{|c|}{\textbf{Literal}} & \multicolumn{1}{|c|}{1} & \multicolumn{1}{|c|}{0} & \multicolumn{1}{|c|}{1} & \multicolumn{1}{|c|}{0} \\
		\cline{1-6}
		\multicolumn{1}{|c|}{\emph{Include}} & \multicolumn{1}{|c|}{P(Reward)} & \multicolumn{1}{|c|}{0} & \multicolumn{1}{|c|}{-} & \multicolumn{1}{|c|}{0} & \multicolumn{1}{|c|}{0} \\ [2ex]
		\cline{2-6}
		\multicolumn{1}{|c|}{} & \multicolumn{1}{|c|}{P(Inaction)} & \multicolumn{1}{|c|}{1.0} & \multicolumn{1}{|c|}{-} & \multicolumn{1}{|c|}{1.0} & \multicolumn{1}{|c|}{1.0} \\ [2ex]
		\cline{2-6}
		\multicolumn{1}{|c|}{} & \multicolumn{1}{|c|}{P(Penalty)} & \multicolumn{1}{|c|}{0} & \multicolumn{1}{|c|}{-} & \multicolumn{1}{|c|}{0} & \multicolumn{1}{|c|}{0} \\ [2ex]
		\cline{1-6}
		\multicolumn{1}{|c|}{\emph{Exclude}} & \multicolumn{1}{|c|}{P(Reward)} & \multicolumn{1}{|c|}{0} & \multicolumn{1}{|c|}{0} & \multicolumn{1}{|c|}{0} & \multicolumn{1}{|c|}{0} \\ [2ex]
		\cline{2-6}
		\multicolumn{1}{|c|}{} & \multicolumn{1}{|c|}{P(Inaction)} & \multicolumn{1}{|c|}{1.0} & \multicolumn{1}{|c|}{0} & \multicolumn{1}{|c|}{1.0} & \multicolumn{1}{|c|}{1.0} \\ [2ex]
		\cline{2-6}
		\multicolumn{1}{|c|}{} & \multicolumn{1}{|c|}{P(Penalty)} & \multicolumn{1}{|c|}{0} & \multicolumn{1}{|c|}{1.0} & \multicolumn{1}{|c|}{0} & \multicolumn{1}{|c|}{0} \\ [2ex]
		\cline{1-6}
	\end{tabular}
	\caption{Type II Feedback}
	\label{tbl:feed2}
\end{table}


\subsection{Tsetlin Machines: An Example Run}

For a didactic example, consider a simple TsM with $2$ monomials in the (binary) sentiment classification task. 
Given an input vector $X$, the output can be classified as a positive sentiment 
or as a negative sentiment. 
During the pre-processing step, $5$ features 
$$[\sf great, boring, bad, interesting, truly]$$ 
have been chosen 
to describe the sentiment. The features are binarized according to their presence in the sentence: if the word occurs in the sentence, the value is set to 1 and 0 otherwise. The binary vector is used to train the model. 

\noindent The features form the literal set: 
$$L=\{\sf great, boring, bad, interesting, truly,$$
$$\sf
\neg{great} , \neg{boring}, \neg{bad} , \neg{interesting}, \neg{truly}
\}$$

\noindent The hyperparameter $N$ is set to 5 which means that each automaton has $2N$ states. Figure \ref{run0} shows the initial memory state for both a positive monomial $C^+$ voting for positive classification, and anegative monomial $C^-$ voting for negative classification. The initial memory state is set to $5$ for all literals, which means that all are equally ``Forgotten''. 

\begin{figure}
	\centering  
	\subfloat[\centering Positive monomial $C^+$]{{\includegraphics[scale=0.2]{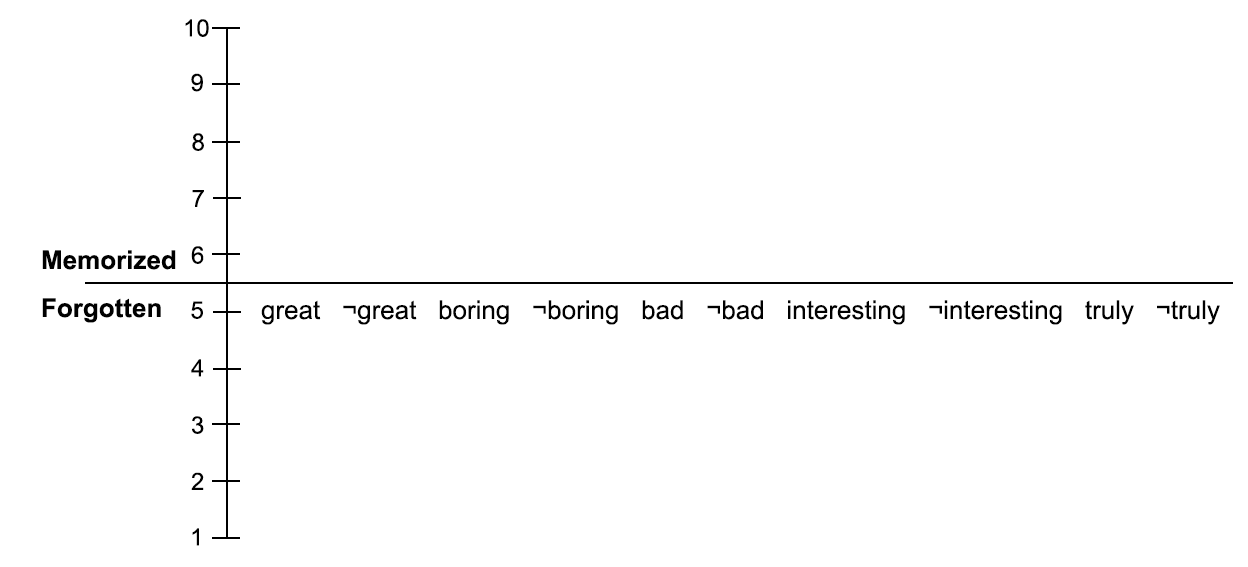} }}
	\qquad
	\subfloat[\centering Negative monomial $C^-$]{{\includegraphics[scale=0.2]{0.png} }}
	\caption{Initialization of the example run}
	\label{run0}
\end{figure}

\noindent Given the first input: \\

\emph{This movie was \textbf{truly} \textbf{great} and \textbf{interesting}} $\rightarrow$ \emph{Positive} \\

\noindent This sentence will produce an input vector $X_1=[1, 0, 0, 1, 1]$ and the label $y_1$ is positive. All literals are in the ``Forgotten'' state, so there are no literals to vote for any of the monomials. 
Empty monomials can be defined in different ways and,
for simplicity of this example, we define both of them as $1$, that is, $C^+ = 1$ and $C^- = 1$. We also skip the random selection of the monomials to receive feedback. The positive label $y_1$ for $X_1$ makes $C^+$ receive Type I Feedback and $C^-$ receive Type II Feedback. 
We have that $C^+$ votes for the true positive output class and it is boosted by Type I Feedback. In contrast, $C^-$ votes for false negative output and it is handled by Type II Feedback, which penalizes negative literals. Figure \ref{run1} illustrates literal evaluation for both monomials after receiving feedback. The black arrows indicates the high probability (1 or $\dfrac{s-1}{s}$), while gray arrows indicates low probability ($\dfrac{1}{s}$). The inaction feedback is omitted in the figure. 

\begin{figure}
	\centering
	\subfloat[\centering Positive monomial $C^+$]{{\includegraphics[scale=0.2]{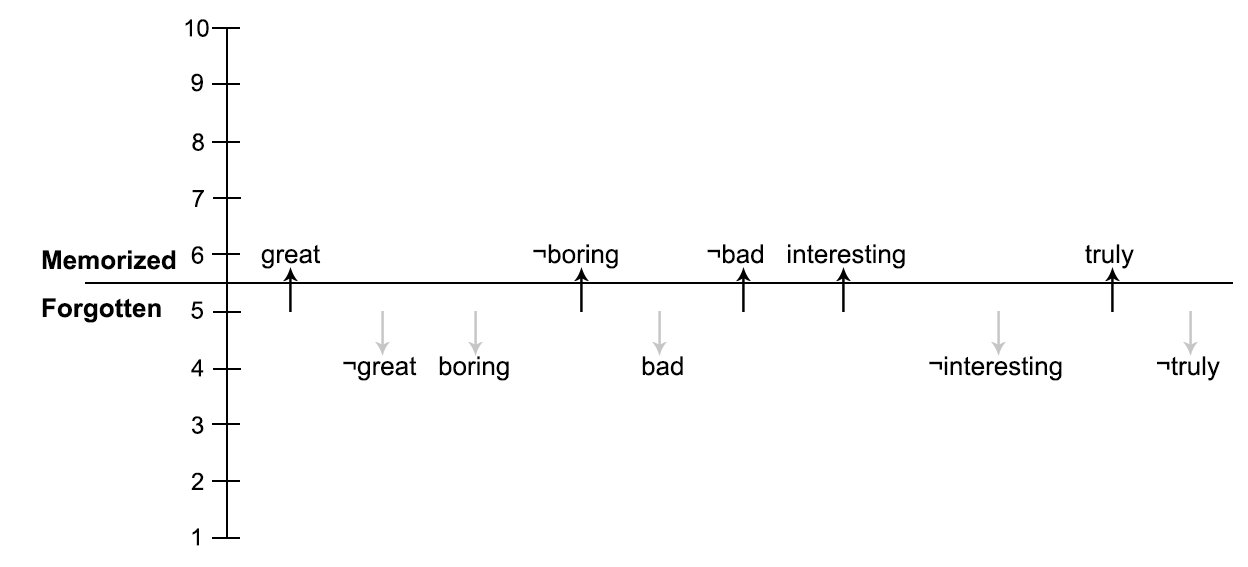} }}
	\qquad
	\subfloat[\centering Negative monomial $C^-$]{{\includegraphics[scale=0.2]{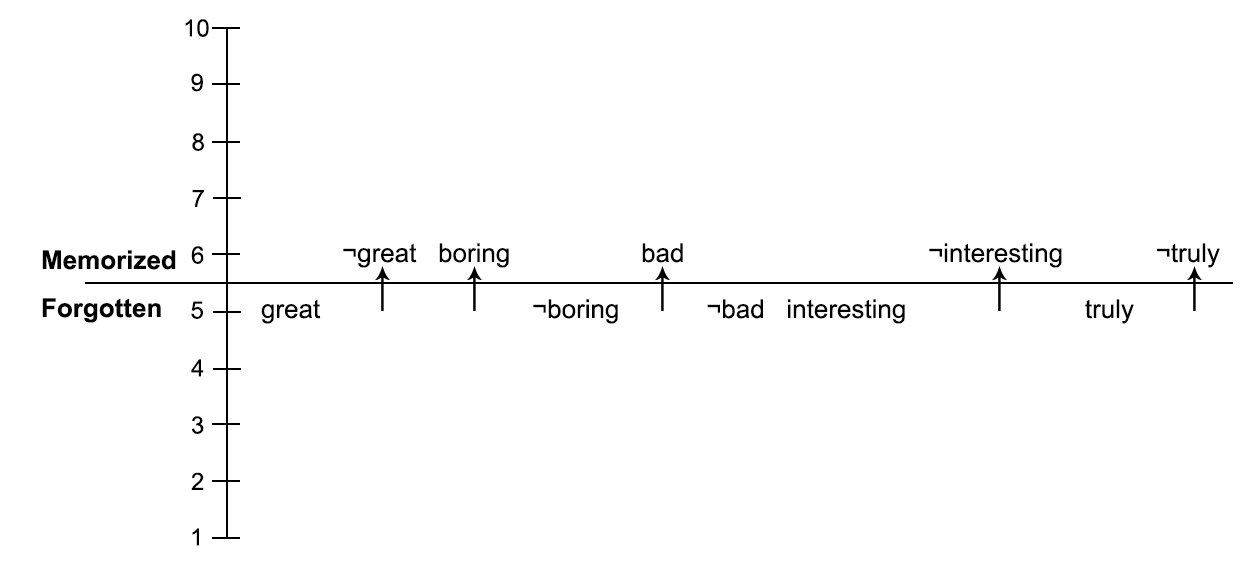} }}
	\caption{Memory update for ``\emph{This movie was \textbf{truly} \textbf{great} and \textbf{interesting}} $\rightarrow$ \emph{Positive}''}
	\label{run1}
\end{figure}

\noindent Next, let the second input be: \\

\emph{\textbf{Truly} \textbf{boring} and \textbf{bad} movie} $\rightarrow$ \emph{Negative} \\

\noindent This sentence will produce an input vector $X_2=[0, 1, 1, 0, 1]$ and the label $y_2$ is negative. Figure \ref{run1} illustrates the ``Memorized'' literals voting for positive classification, which are\footnote{We write $\overline{p}$ for the negation of a positive literal $p$, which in propositional logic is $\neg p$.}
$$C^+ = (\sf great \cdot \overline{ boring} \cdot \overline {bad} \cdot interesting \cdot truly)$$ 
and ``Memorized''  literals voting for negative, which are $$C^- = (\sf \overline{great}  \cdot boring \cdot bad \cdot \overline{interesting}\cdot  \overline{truly} ).$$ Feeding the input $X_2$ to the monomials results in $C^+ = 0 \cdot 0 \cdot 0 \cdot 0 \cdot 1 = 0 $ and $C^- = 1 \cdot 1 \cdot 1 \cdot 1 \cdot 0 = 0$. The label $y_2$ is negative 
for $X_2$, therefore $C^+$ receives Type II Feedback and $C^-$ receives Type I Feedback. Type II Feedback penalizes only false positive outputs, so for $C^+$ 
the output is a true negative and all literals remain unchanged as seen in Figure \ref{run2}. Type I Feedback is combating false negative output, that is,
when $C^-$ is \emph{not} trigged and the label $y_2$ is negative. 
The literals $\sf great, \neg boring, \neg bad, interesting, truly$ are in state  $a_{i\in\{2,3,5,8,10\}}=6$,
so there is an \emph{include} action in Type I Feedback. Both positive and negative literals get penalized with equally low probability. The remaining literals are in state $a_{i\in\{1,4,6,7,9\}}=5$ meaning they perform \emph{exclude} action which rewards all literals with low probability as well. Considering the probability, assume that only the literals marked with black color for the negative monomial in Figure \ref{run2} changed their state $a_{i\in\{1,4,10\}}$ by -1 (that is, $\sf great, \neg boring, \neg truly$,). The rest of the literals remained in the previous state of Figure \ref{run1}.

\begin{figure}
	\centering
	\subfloat[\centering Positive monomial $C^+$]{{\includegraphics[scale=0.2]{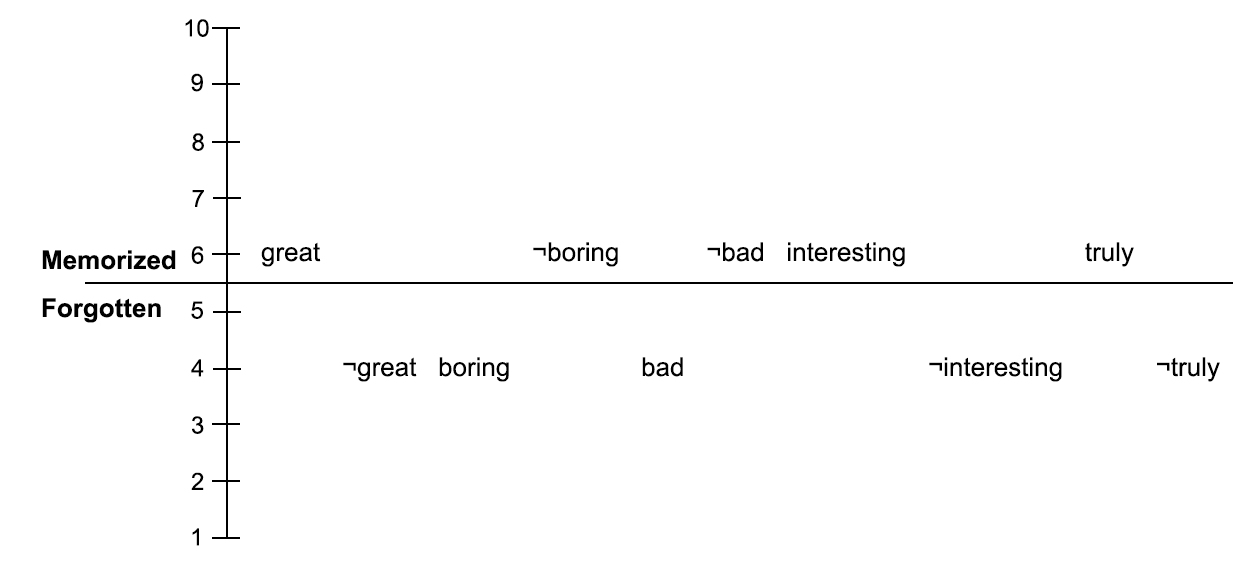} }}
	\qquad
	\subfloat[\centering Negative monomial $C^-$]{{\includegraphics[scale=0.2]{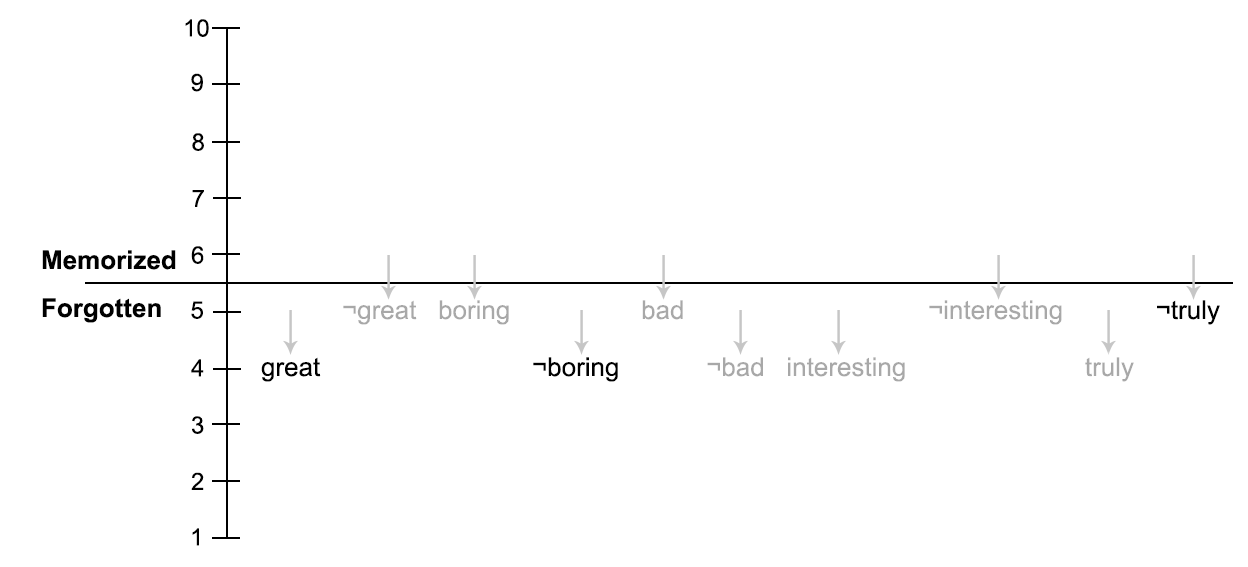} }}
	\caption{Memory update for ``\emph{\textbf{Truly} \textbf{boring} and \textbf{bad} movie} $\rightarrow$ \emph{Negative}''}
	\label{run2}
\end{figure}

\begin{figure}
	\centering
	\subfloat[\centering Positive monomial $C^+$]{{\includegraphics[scale=0.19]{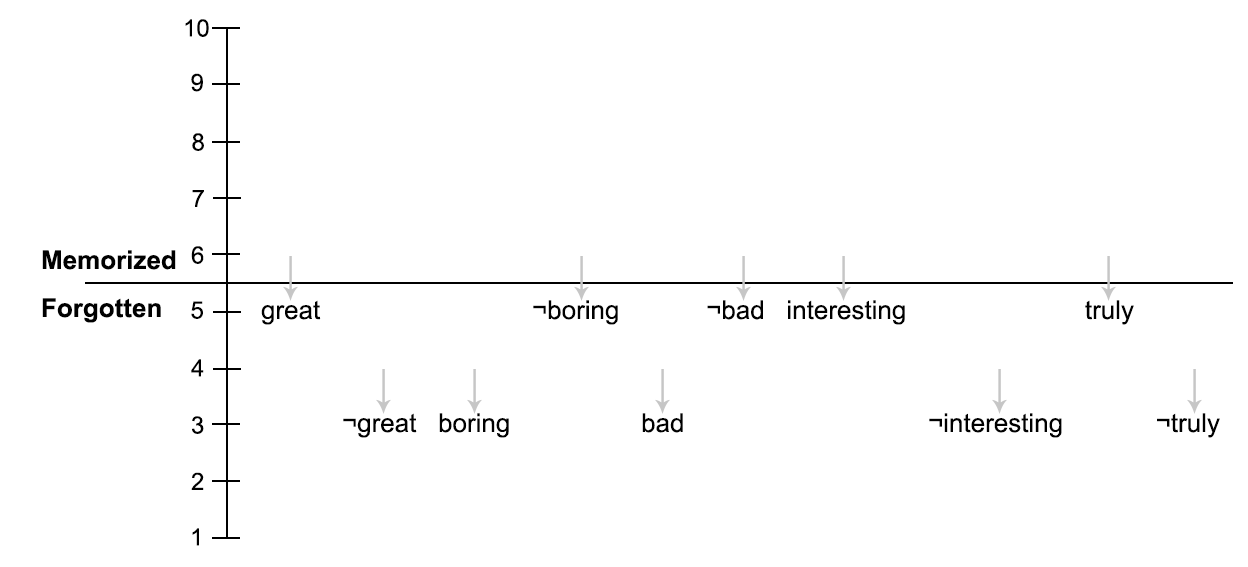} }}
	\qquad
	\subfloat[\centering Negative monomial $C^-$]{{\includegraphics[scale=0.19]{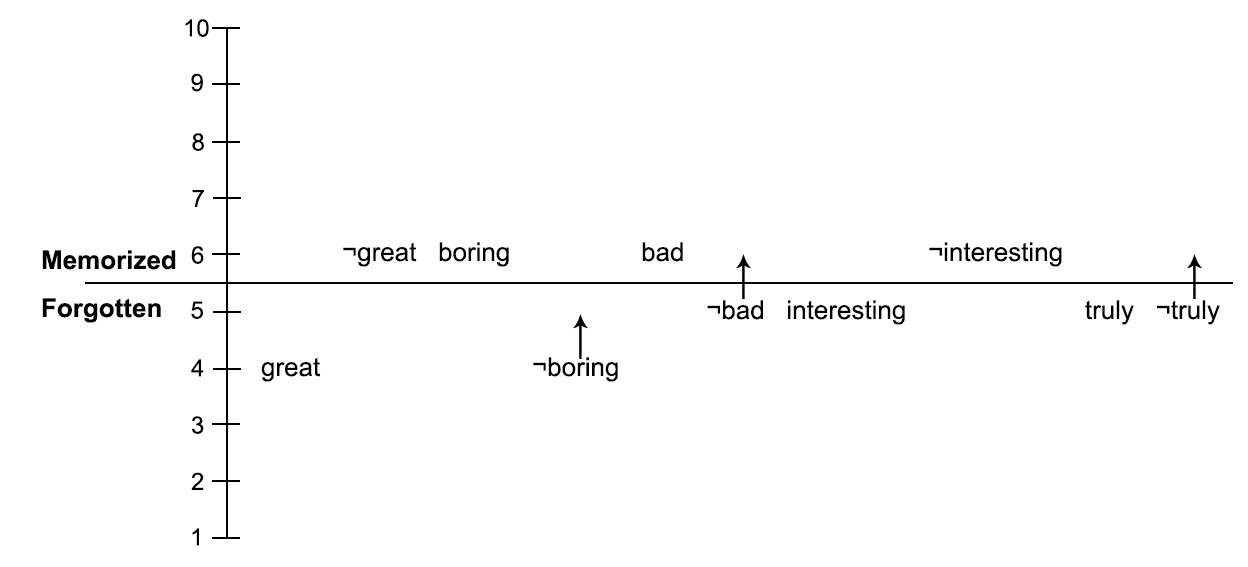} }}
	\caption{Memory update for ``\emph{I thought this movie was going to be \textbf{bad} and \textbf{boring} but it was \textbf{truly} good} $\rightarrow$ \emph{Positive}''}
	\label{run3}
\end{figure}

\noindent Now, let the third input be: \\

\emph{I thought this movie was going to be \textbf{bad} and \textbf{boring} but it was \textbf{truly} good} $\rightarrow$ \emph{Positive} \\

\noindent This sentence will produce an input vector $X_3=[0, 1, 1, 0, 1]$ with the positive label $y_3$. The ``Memorized'' literals that vote for the negative monomial are now $$C^- = (\sf \overline{ great} \cdot boring \cdot bad \cdot \overline{ interesting}).$$ The negative monomial will evaluate to $C^- = 1 \cdot 1 \cdot 1 \cdot 1 = 1$ being a false positive. Literals voting for $C^+$ are the same as previously, that is, $$C^+ = (\sf great \cdot \overline {boring} \cdot \overline{bad}   \cdot interesting \cdot truly).$$ Given the input vector, $X_3$ the positive monomial is $C^+ = 0 \cdot 0 \cdot 0 \cdot 0 \cdot 1 = 0$, which implies that it will not vote for a positive classification. We then have a  false negative. The positive label $y_3$ implies that Feedback Type I is provided to $C^+$, and Feedback Type II is given to $C^-$. Type I Feedback penalizes all of the literals being in the \emph{include} state and rewards literals having \emph{exclude} state, meaning all of the literals in the positive monomial decrements their states with a low probability of $1/s$. Negative monomials should seek to return 0 for the input $X_3$, but they do otherwise. All literals that are in the 
\emph{exclude} zone, that is, states below or equal to 5, and whose value is 0 on the input ($\neg \sf boring, \neg bad, \neg truly$) are penalized with a probability of 1 by the Type II Feedback (Figure~\ref{run3}).

This finishes our illustrated presentation of 
how a TsM learns patterns from binary classified inputs.
In the next section we provide proofs for the theorems 
in the main text.
\section{Proofs for Section III}

\theoremtwo*
\begin{proof}[Sketch]
	Suppose there is an interpretation \Jmc 
	that satisfies 	
	$ \notrob(\Mmc,\Imc,\epsilon) $. In this case we want to show that \Mmc is \emph{not} $\epsilon$-robust for \Imc. 
	Since Eq.~\ref{rob:1} 
	is satisfied, we have $ \hamming_n(\Imc,\Jmc) \leq \epsilon $,
	where $n$ is the dimension of the TsM with formula \Mmc.
	By Theorem~\ref{thm:sat}, the expression	
	in Eq.~8	is satisfiable when the TsM formula classifies $\Jmc$ differently from the classification of \Imc
	(recall that we write 
	$\Mmc_{\vec{x}}$ instead of \Mmc just to make explicit the use of variables in $\vec{x}$).
	%
	By Definition~\ref{p1}
	this holds iff
	$\Mmc$ is not $ \epsilon $-robust for \Imc.
	
	Conversely, if $ \notrob(\Mmc,\Imc,\epsilon) $ is not satisfiable
	then there is no interpretation \Jmc which satisfies
	Eq.~\ref{rob:1} and Eq.~\ref{rob:xor} (in other words, 
	with hamming distance at most $\epsilon$ from \Imc)
	and is classified differently from $\Mmc_{\vec{x}}(\Imc)$.
	This means that $\Mmc_{\vec{x}}$ is $ \epsilon$-robust for $ \Imc $.
	%
\end{proof}

\theoremthree*
\begin{proof}
	By Theorem~\ref{thm:robust},
	each $ c_i $ in Eq.~9 
	is set to true iff the TsM \Mmc is not 
	$ \epsilon $-robust for the given   example $ \Imc_i $.
	Therefore, Eq.~\ref{unirob:2} is satisfied iff 
	the number of   examples that do not pass the $ \epsilon $-robustness condition
	is above the  threshold $ \lfloor\eta |S|\rfloor $. 
\end{proof}

\theoremfour*
\begin{proof}
	By construction of 
	$ \encts(\Mmc_i) $, for $ 1\leq i \leq 2 $,
	the variables in $ \vec{x} $
	correspond to the input interpretation of the 
	\ac{TsM} $ \Mmc_i $.
	By Theorem~\ref{thm:sat} and since $ \vec{x} $
	is shared between the encodings of the two \acp{TsM},
	$ 	\Mmc_i(\Imc)=1 $ 
	iff $ \encts(\Mmc_i)_{[\vec{x}\rightarrow {\Imc}]}\wedge o^i  $ is satisfiable.
	The formula $(\encts(\Mmc_1) \wedge o^1)\leftrightarrow (\encts(\Mmc_2)\wedge o^2)$
	is falsified iff it is possible to find an \Imc
	which is classified differently by the \acp{TsM}.
\end{proof}

\theoremfive*
\begin{proof}
	This theorem can be proven with arguments similar to those used for robustness and universal robustness.
\end{proof}

\end{document}